\documentclass[11pt]{article} 
\usepackage[left=1 in, right=1 in, top=1 in, bottom = 1 in]{geometry}
\usepackage{amsmath}
\usepackage{hyperref}
\usepackage{graphicx}
\usepackage{indentfirst}
\usepackage{amssymb}
\usepackage[capitalise]{cleveref}
\usepackage{mathtools}
\usepackage{amsthm}
\usepackage{mathtools}
\usepackage[title]{appendix}
\usepackage{xcolor}
\usepackage{comment}
\usepackage{bbm}
\usepackage{amssymb,amsthm, amsfonts}
\usepackage{algorithm, algorithmicx, algpseudocode, graphicx}
\usepackage{wrapfig}
\newtheorem{theorem}{Theorem}
\newtheorem{lemma}{Lemma}
\newtheorem{definition}{Definition}
\newtheorem{remark}{Remark}

\newtheorem{corollary}{Corollary}

\newcommand{\A}{\mathcal{A}}
\newcommand{\E}{\mathbb{E}}
\newcommand{\M}{\mathcal{M}}
\newcommand{\F}{\mathcal{F}}
\newcommand{\sigmac}{\sqrt{\frac{c}{k+b+1}}}
\newcommand{\D}{\mathcal{D}}
\newcommand{\Eimu}{E_i ^{\mu} (t)}

\newcommand{\Eitheta}{E_i^{\theta} (t)}

\newcommand{\N}{\mathcal{N}}
\newcommand{\pit}{p_{i,t}}
\newcommand{\armt}{a_t}
\newcommand{\rit}{r_t}

\newtheorem*{lemma2}{Lemma 2}
\newtheorem*{lemma4}{Lemma 4}
\newcommand{\nitcap}{n_{i,T}}
\newcommand{\muit}{\hat{\mu}_{i,t}}

\newcommand{\muitminus}{\hat{\mu}_{i,t-1}}
\newcommand{\nitminus}{n_{i,t-1}}
\newcommand{\thetait}{\theta_{i,t}}

\newcommand{\mustar}{\mu^{\ast}}

\algrenewcommand\algorithmicrequire{\textbf{Input:}}
\algrenewcommand\algorithmicensure{\textbf{Output:}}

\usepackage{xcolor}
\usepackage{siunitx}

\usepackage{thmtools} 
\usepackage{thm-restate}

\definecolor{DarkGreen}{rgb}{0.1,0.5,0.1}

\begin{document}
\title{Thompson Sampling Itself is Differentially Private}
\author{Tingting Ou\footnotemark[1] \and Marco Avella Medina\footnotemark[1] \and Rachel Cummings\footnotemark[1]} 

\renewcommand{\thefootnote}{\fnsymbol{footnote}}
\footnotetext[1]{Columbia University. Emails: \texttt{\{to2372,ma3874,rac2239\}@columbia.edu} T.O. and R.C. were supported in part by NSF grants CNS-2138834 (CAREER) and EEC-2133516. M.A.M. was supported in part by NSF grant DMS-2310973.}

\renewcommand{\thefootnote}{\arabic{footnote}}

%\date{}
\maketitle

\begin{abstract}

In this work we first show that the classical Thompson sampling algorithm for multi-arm bandits is differentially private as-is, without any modification. We provide per-round privacy guarantees as a function of problem parameters and show composition over $T$ rounds; since the algorithm is unchanged, existing $O(\sqrt{NT\log N})$ regret bounds still hold and there is no loss in performance due to privacy.  
We then show that simple modifications -- such as pre-pulling all arms a fixed number of times, increasing the sampling variance -- can provide tighter privacy guarantees. We again provide privacy guarantees that now depend on the new parameters introduced in the modification, which allows the analyst to tune the privacy guarantee as desired. We also provide a novel regret analysis for this new algorithm, and show how the new parameters also impact expected regret. 
Finally, we empirically validate and illustrate our theoretical findings in two parameter regimes and demonstrate that tuning the new parameters substantially improve the privacy-regret tradeoff.

\end{abstract}

\section{Introduction}\label{s.intro}

The Thompson Sampling algorithm is one of the earliest developed heuristics for the stochastic multi-arm bandits problem and has been proven to have good performance both empirically and theoretically. 
It is a Bayesian regret-minimization algorithm that initializes a prior distribution on the parameters of the reward distributions, plays the arm using the posterior probability of being the best arm, and updates the posterior distribution accordingly using the observations. It appeared for the first time by \cite{T1933} in a two-armed bandit problem motivated by clinical trials. After being largely ignored in the literature, strong empirical performance and theoretical guarantees led to a rapid surge of interest in this algorithm in the last 15 years \cite{russoetal2018}. Indeed, the Thompson Sampling algorithm has since been widely studied and proven to be useful for solving a wide range of online learning problems \cite{AG12,AG17,AG13,WC18,russoetal2018,LR21,ZYJ21,HAT20}.

In this work, we first analyze the privacy guarantees of Thompson Sampling to show that the classical Thompson sampling algorithm for multi-arm bandits is differentially private (DP) as-is, without any modification. The algorithm determines the next arm to play at each timestep by first sampling an estimate of each arm's mean reward from the posterior, and then selecting the arm with the highest noisy posterior sample. When the algorithm is initialized with Gaussian priors on reward distributions, this step is equivalent to adding mean-zero Gaussian noise to the empirical mean of realized rewards, also known as the \emph{Gaussian Mechanism} in the DP literature.

We provide per-round privacy guarantees as a function of problem parameters and show composition over $T$ rounds for $N$ arms. In the main body, we express per-round privacy in terms of Gaussian differential privacy (GDP) \cite{DRS19}, compose GDP parameters across all rounds, and then translate the guarantees back to the standard DP parameters. GDP is known to be amenable to many rounds of composition and the addition of Gaussian noise, both of which occur in the Thompson Sampling algorithm, but it leads to less easily interpretable statements of DP guarantees. In Appendix \ref{app.otherprivacy}, we provide an alternative proof of the DP guarantees of Thompson Sampling, that relies on a direct analysis of the DP guarantees. Along the way, we also show that a more general version of the classic ReportNoisyMax algorithm \cite{DR14} with heterogeneous Gaussian noise  still satisfies DP, which may be of independent interest. Since the algorithm is unchanged, existing regret bounds from non-private analysis of Thompson Sampling \cite{AG17} still hold, and there is no loss in performance due to privacy.

Next, we show that simple modifications -- such as pre-pulling all arms $b$ times or increasing the sampling variance by a factor of $c$ -- can provide even tighter privacy guarantees, and allow the analyst to tune the privacy parameters as desired. The analysis follows a similar proof structure as before -- proving per-round GDP guarantees, composing across $T$ rounds, and translating back to DP -- but this time accounting for the impact of the new parameters $b$ and $c$. Since we modify the Thompson Sampling algorithm, existing regret bounds no longer hold and must be re-derived. We provide a novel regret analysis for our algorithm that follows the same high-level structure of the original analysis of \cite{AG17}, while tightening some intermediate steps and tracking the impact of the new parameters in the expected regret. 

Finally, we empirically validate our theoretical findings for two different families of reward distributions: Bernoulli and truncated exponential.  For both families, our experimental findings match our theoretical results: tuning $b$ and $c$ can lead to substantial improvements in the empirical regret, for the same fixed privacy guarantee. We also observe that the optimal tuning strategy in terms of the privacy-accuracy tradeoff appears to involve jointly tuning $b$ and $c$, rather than tuning just one of the parameters.

\subsection{Related Work}\label{s.relwork}

There is a large body of work that considers the setting of private stochastic multi-arm bandits: \cite{MT15,HH22,TD16,SS19,HU21}. These works have considered different algorithms for the problem of privately learning from bandits, adapting popular non-private procedures such as Successive Elimination, Upper Confidence Bound and Thompson Sampling. 

We focus our attention in this work on the Thompson Sampling algorithm \cite{T1933} which has been shown to to be applicable in solving a wide range of online learning problems including the classic multi-arm bandits problem \cite{AG12,AG17}, contextual bandits \cite{AG13}, combinatorial semi-bandits \cite{WC18} and other applications including online job scheduling, subset-selection, variable selection, opportunistic routing, combinatorial network optimization \cite{GMM14,LR21,ZYJ21,HAT20}. 

The empirical efficacy of Thompson Sampling when applied to the stochastic multi-arm bandits problem was demonstrated \cite{COL11} before the currently best known theoretical regret bounds were proven \cite{AG12}. The seminal work on Thompson Sampling \cite{AG12} gave a regret bound of $O( (\sum _{i=2}^{N}\frac{1}{\Delta_i})^2 \log T )$ for the stochastic multi-arm bandits problem with $N$ arms over $T$ timesteps, when the algorithm is instantiated with a Beta prior over rewards, where $\Delta_i$ is the reward gap between the best arm and arm $i$. When the $\Delta_i$ are all bounded away from 0, this gives a problem-independent regret bound of $O(N^2 \log T)$. 
Follow-up work \cite{AG17} gives an $O(\sqrt{N T\log T})$ regret bound for the Thompson Sampling algorithm for both a Beta prior and a Gaussian prior. Our work focuses on the Gaussian prior version of the Thompson Sampling algorithm and its privacy guarantees. 

Some recent works on private stochastic bandits have specifically addressed the problem of privatizing the Thompson Sampling algorithm as a solution to the bandits problem. \cite{MT15} first designed an $\epsilon$-differentially private variant of the Thompson sampling algorithm, with expected regret $O(N\frac{\log ^3 T}{\Delta^2 \epsilon^2})$, where $\Delta$ is the reward gap between the best arm and the second best arm. \cite{HH22} presented two $\epsilon$-differentially private (near)-optimal Thompson Sampling-based algorithms, DP-TS and Lazy-DP-TS, with regret bounds $\sum _j O(\frac{\log T }{\min \{ \epsilon, \Delta_j\}} \log (\frac{\log T}{ \epsilon \Delta_j}))$ and $\sum _j O(\frac{\log T }{\min \{ \epsilon, \Delta_j\}} )$, respectively. Both works involve significant modifications of the Thompson Sampling algorithm to guarantee privacy, while our work mainly analyzes the privacy guarantee of the \textit{original} version of the Thompson Sampling algorithm.  

This work builds upon works analyzing the privacy guarantees of existing well-studied randomized algorithms, most notably \cite{BBDO12}, which showed that the Johnson-Lindenstrauss transform itself preserves differential privacy, and \cite{SSG20}, which proved that the Flajolet-Martin Sketch itself is differentially private. Similar in spirit -- although completely different in technical details -- we show that the noise added in Thompson Sampling is sufficient to satisfy differential privacy.

\section{Model and Preliminaries}\label{s.prelim}

We consider the classic stochastic multi-armed bandit (MAB) setting with $N$ arms. At each time $t \in [T]$, an arm $i \in \A = [N]$ is chosen to be played based on the outcomes from the previous $(t-1)$ timesteps, and yields a random real-valued reward $\rit \in [0,1]$ sampled from a fixed unknown distribution $\D_i$ with mean $\mu_i$. The rewards obtained from playing any arm are sampled i.i.d.~and are independent of time or the plays of other arms. 

The analyst must specify a \emph{policy} $\pi = \{\pi_t\}_{t\in [T]}$, where each $\pi_t$ maps from the history $\mathcal{F}_{t-1} = \{(a_{\tau},r_{\tau})\}_{\tau<t}$ containing the sequence of arms played and rewards realized up to time $t-1$, to the new arm $\armt$ played at time $t$.

We measure the analyst's success using the standard metric of expected total \emph{regret}, which is the difference between the best possible expected cumulative reward (if the distributions $\D_i$ were known) and the expected total reward under policy $\pi$ \cite{AG13}. Letting $\mustar = \max _i \mu_i$, the expected total regret is:
\begin{equation}\label{eq.regret}
\E[\mathcal{R}(T, \pi)] = \mustar T - \E_{\armt \leftarrow \pi}[\textstyle\sum_{t=1}^T r_{t}].
\end{equation}

For the remainder of this paper, $\log$ will always refer to the natural logarithm.

\subsection{Thompson Sampling}\label{s.tsprelims}

The Thompson Sampling algorithm is a commonly used policy for minimizing regret in a MAB setting \cite{T1933,AG12,russoetal2018}. One reason for the widespread use of Thompson Sampling is that it is known to achieve low regret.  At a high-level, the algorithm operates as follows. It starts with a prior belief on the mean reward $\mu_i$ for each arm. After observing each reward $r\sim \D_i$, the algorithm performs a Bayesian update to compute a posterior distribution of $\mu_i | r$. At every timestep, the algorithm samples a value $\theta_i$ for each arm according to the posterior given all realized rewards, and then plays the arm with the highest sampled $\theta$ value.

The Thompson Sampling algorithm is parameterized only by its initial priors on $\D_i$. In this work we focus on the special case of Gaussian priors, where the algorithm is initialized with a Gaussian prior for each arm. We emphasize that the algorithm's priors on rewards need not match the true reward distributions, and thus this does not conflict with our modeling assumption of bounded rewards. These priors are a part of the algorithmic construction, and not a part of the underlying data model.

We focus on the single parameter setting where each reward distribution $\D_i$ has unknown mean $\mu_i$ and known variance $\sigma_i^2 = 1$. For each arm $i$, the algorithm is initialized with a prior belief that $\mu_i \sim \mathcal{N}(0,1)$; after observing a reward, the posterior distribution of each $\mu_i$ will remain Gaussian. Concretely, given the initial prior $\mathcal{N}(0,1)$ and a sequence of rewards $(r_1, \ldots, r_t)$, the posterior for arm $i$ from which $\theta_i$ will be sampled is $\mathcal{N}(\hat{\mu}_{i,t}, \frac{1}{n_{i,t}+1})$, where $n_{i,t}$ is the number of times arm $i$ is played up to time $t$, and $\hat{\mu}_{i,t} = \frac{1}{n_{i,t}+1}\sum_{\tau = 1}^t 1_{a(\tau)=i} r_{\tau}$ is the empirical mean of realized rewards from arm $i$, with a slight offset to avoid degeneracy based on the initial prior. 

The $\theta_i$ for each arm is sampled according to this posterior at each timestep, and the arm corresponding to the largest realized $\theta$ is selected, output, and the reward from the selected arm is observed internally by the algorithm. Note that this reward is not a part of the algorithm's output.
Algorithm \ref{alg.ts.orig} presents a formal description of this algorithm.

\begin{algorithm}
\begin{algorithmic}[1]
\caption{Thompson Sampling with Gaussian priors}\label{alg.ts.orig}
\State \textbf{Input:} number of arms $N$, time horizon $T$
\State Initialize $\hat{\mu}_{i, 0} = 0, \ n_{i,0} = 0$ for each $i = 1, \ldots, N$
\For {$t = 1, 2, \ldots, T$}
    \State For each arm $i = 1, \ldots, N$, sample independently $\theta_{i,t} \sim \N(\hat{\mu}_{i,t-1}, \frac{1}{n_{i,t-1}+1})$
    \State Play arm $a_t : = \arg \max _i \theta_{i,t}$ and realize $r_t$
    \State \textbf{Output} $a_t$
    \State For $i=a_t$, update $ \hat{\mu}_{i, t} = \frac{\hat{\mu}_{i, t-1}  (n_{i,t-1} + 1) + r_t}{n_{i, t-1} + 2}$, and $ n_{i,t} = n_{i,t-1} + 1$
    \State For all $i \neq a_t$, update $ \hat{\mu}_{i, t} =\hat{\mu}_{i, t-1}$, and $ n_{i,t} = n_{i,t-1}$
\EndFor
\end{algorithmic}
\end{algorithm}

Regret as defined in Equation \eqref{eq.regret} provides a \emph{problem-independent} definition, because it has no additional assumptions or dependence on the problem instance characterized by the true reward means $(\mu_1,\dots,\mu_N)$.
%\mam{it is not really clear to me what we mean by this. I think we mean the values of $(\mu_1,\dots,\mu_N)$?}. 
We will also consider an equivalent \emph{problem-dependent} definition of regret, which provides guarantees based on the true gap between the arm means. Define the suboptimality gap of arm $i$ to be $\Delta_i := \mu^* - \mu_i$. We can use this, and the random variable $n_{i,t}$, which denotes the number of times that arm $i$ has been played up to time $t$, to re-write the expected regret in a problem-dependent manner:
\begin{equation}\label{eq:dependentregret}
    \E[\mathcal{R}(T, \pi)] = \E_{\armt \leftarrow \pi}[\sum_{t=1}^T(\mu^*-\mu_{a_t})]=\sum_{i=1}^N\Delta_i\E_{\armt \leftarrow \pi}[n_{i,T}].
\end{equation}

\subsection{Differential Privacy}\label{s.dpprelims}

Differential privacy (DP) is a parameterized notion of database privacy, which ensures that changing a single element of the input database will lead to only small changes in the distribution over outputs. In our online setting, we consider database \emph{streams}, where each data element (from a data universe $\mathcal{R}$) arrives one-by-one, and the algorithm must produce an output from some output space $\mathcal{O}$ at each timestep. Two length $T$ streams $R,R' \in \mathcal{R}^T$  are said to be \emph{neighboring} if they differ only in the data element received in a single timestep.  We first recall the definition of differential privacy for streams, adapted from its standard presentation of \cite{DMNS06} to the streaming setting as in \cite{CPW16}. 

\begin{definition} A streaming algorithm $\M:\mathcal{R}^T \to \mathcal{O}^T$ is $(\epsilon, \delta)$-differentially private if for any pair of neighboring streams $R,R' \in \mathcal{R}^T$ and for any set of outputs $S \subseteq \mathcal{O}^T$, 
\[\Pr [ \M(R) \in S] \leq e^\epsilon \Pr [ \M(R') \in S] + \delta. \] 
\end{definition}

In the context of Thompson Sampling, our \emph{neighboring streams} correspond to sequences of rewards $R=\{r_{t}\}_{t\in[T]}$ and $R'=\{r'_{t}\}_{t\in[T]}$ that differ in a single reward value: there exists a single $\tau$ such that $r_{\tau} \neq r'_{\tau}$, and for all other $t \neq \tau$, $r_t=r'_t$. The output $\mathcal{O}^T$ is the sequence of arm pulls $\{a_{t}\}_{t\in[T]}$ output by the algorithm. 

Note that streaming algorithms can be \emph{adaptive}, where the chosen arm $a_t$ at time $t$ is a function of all previous rewards and arm pulls.
We will also analyze the privacy of single-shot mechanisms, where the mechanisms output space is $\mathcal{R}$ instead of $\mathcal{R}^T$. This corresponds to analysis at a single fixed timestep, rather than across all time.

We will use Gaussian differential privacy (GDP) \cite{DRS19}, which is a variant of DP that is more amenable to many rounds of composition -- where private algorithms are applied many times to the same dataset -- and the addition of Gaussian noise, both of which occur in the Thompson Sampling algorithm.  
First we see that it is easy to translate between GDP and the standard $(\epsilon, \delta)$-DP notion.

\begin{definition}[\cite{DRS19}] For any two probability distributions $P$ and $Q$ on the same space, define the trade-off function $T(P,Q) : [0,1] \rightarrow [0,1]$  as $T(P,Q) (\alpha) = \inf \{ \beta _{\phi} : \alpha_{\phi} \leq \alpha \}$ where the infimum is taken over all (measurable) rejection rules. Then, a mechanism $\M$ satisfies $\eta$-GDP if for all neighboring databases $S$ and $S'$, $T(\M(S), \M(S')) \geq T(\N(0,1), \N(\mu, 1))$. 
\end{definition}

GDP is known to be translatable with DP (Lemma \ref{lem.gdp-to-eps}), and it \emph{composes} (Lemma \ref{lem.GDP-comp}), meaning that the privacy parameter $\mu$ composes slowly as more computations are performed on the data.

\begin{lemma}[\cite{DRS19}]\label{lem.gdp-to-eps}
A mechanism $\mathcal{M}$ is $\eta$-GDP if and only if it is $(\epsilon, \delta(\epsilon) ) $-DP for all $\epsilon \geq 0 $, where $\delta(\epsilon) = \Phi(-\frac{\epsilon}{\eta} + \eta / 2) - e^{\epsilon} \Phi(-\frac{\epsilon}{\eta} - \eta / 2)$, where $\Phi(x) = \frac{1}{\sqrt{2\pi}} \int _{-\infty}^x \exp(-u^2 / 2) du$ is the cumulative density function of the standard normal distribution.
\end{lemma}

\begin{lemma}[\cite{DRS19}] \label{lem.GDP-comp}
    Let $\M_t$ be an $\eta_t$-GDP mechanism, for $t=1,\ldots,T$. Then the adaptive $T$-fold composition of all $\M_1, \ldots, \M_T$ is $\sqrt{\sum_{t=1}^T \eta_t ^2 }$-GDP.
\end{lemma}

Our analysis will be based on the \emph{Gaussian Mechanism}, which is a method for privately evaluating a real-valued function $f$, and is defined as:
\[\M(D,\sigma^2) = f(D) + Y, \quad \text{where $Y \sim \N(0, \sigma^2)$.}\]  
The Gaussian Mechanism satisfies $(s_f/\sigma)$-GDP \cite{DRS19}, where $s_f = \max_{R,R' \text{neighbors}} |f(R)-f(R')|$ is the \emph{sensitivity} of $f$, or the maximum change in the function's value between neighboring databases. 

\section{Thompson Sampling is DP}\label{s.tsisdp}

In this section, we show that the original Thompson Sampling algorithm (with Gaussian priors) as presented in Algorithm \ref{alg.ts.orig} is differentially private. Intuitively, differential privacy requires that algorithms make randomized decisions based on the database. In the case of Thompson Sampling, the algorithm is inherently randomized, by selecting the next arm based on the (randomly generated) $\theta_i$ rather than the exact empirical mean of historical play.

To show this formally, we must show that the particular distributions of randomness used in Thompson Sampling satisfy the mathematical requirements of differential privacy. 
To prove this, we will first focus on the privacy guarantees of a single step at a fixed time $t$, and then show how the privacy guarantees composes across $T$ timesteps. Since GDP is known to yield improved composition guarantees relative to regular DP composition and other competing DP-like guarantees such as Renyi-DP \cite{DRS19},
 the analysis will involve computing single-step privacy guarantees using GDP (Lemma \ref{lem.tsgdp-onestep}), then applying GDP composition (Lemma \ref{lem.orig-gdp}), and finally converting back to a DP guarantee (Theorem \ref{thm.tsdp}).
While the formal statement with all parameters is given later in Theorem \ref{thm.tsdp}, the main result can be stated informally as in Theorem \ref{thm.dpinformal} below.

\begin{theorem}[Informal version of Theorem \ref{thm.tsdp}]\label{thm.dpinformal}
    Algorithm \ref{alg.ts.orig} satisfies $(\epsilon,\delta)$-differential privacy.
\end{theorem}

Before formally stating and proving this privacy result, we remark that since differential privacy does not require any changes to the algorithm itself, then all existing regret bounds continue to hold under differential privacy without incurring additional loss. These regret bounds are stated below the Theorem \ref{thm.tsregret}, and are tight under the mild assumption that rewards are bounded in $[0,1]$. 

\begin{theorem}[\cite{AG17}]\label{thm.tsregret}
    The Thompson Sampling algorithm with Gaussian priors (Algorithm \ref{alg.ts.orig}) has regret at most $\sum_{i=1}^N O(\frac{\log T}{\Delta_i} ) $ (problem-dependent), or $O(\sqrt{NT \log N})$ (problem-independent). 
\end{theorem}

Returning to the privacy analysis, at time $t$, Algorithm~\ref{alg.ts.orig} samples $\thetait \sim \N(\muitminus, \frac{1}{\nitminus})$ independently for each arm, and then selects $\armt = \arg \max_i \thetait $.
First, consider the vector of sampled mean-estimates $\{\theta_{i,t+1}\}_{i \in [N]}$ that are generated by the Thompson Sampling algorithm at time $t+1$ given history $\mathcal{F}_{t}$. The single-step algorithm $\M_{TS}(\mathcal{F}_{t})$ corresponding to one step of Thompson sampling can be described as follows:
\[ \M_{TS}(\mathcal{F}_{t}) = \arg\max_i \theta_{i,t+1}, \quad \text{ where } \theta_{i, t+1} \sim \N (\frac{1}{n_{i,t}+1}\sum_{\tau = 1}^t 1_{a(\tau)=i} r_{\tau}, \frac{1}{n_{i,t} + 1}). \]

Our first result is that this single step mechanism satisfies Gaussian differential privacy.

\begin{lemma}\label{lem.tsgdp-onestep}
The mechanism $\M_{TS}(\mathcal{F}_t)$ satisfies $\sqrt{\frac12}$-GDP with respect to realized rewards.
\end{lemma}

The full proof of Lemma \ref{lem.tsgdp-onestep} is deferred to Appendix \ref{app.tsonestep}, but we give a brief proof sketch here for intuition. Each $\theta_{i,t+1}$ can be expressed as $\hat{\mu}_{i}(\F_t)$ plus an independent Gaussian noise term sampled from $\N (0, \frac{1}{n_{i,t} + 1})$. Viewing $\hat{\mu}_{i}(\F_t)$ as the real-valued query on the data, this is simply an instantiation of the Gaussian Mechanism. The query has sensitivity $s \leq \frac{1}{n_{i,t} + 1}$ for arm $i$, since rewards are bounded between 0 and 1, so the empirical means on neighboring reward vectors can differ by at most $\frac{1}{n_{i,t} + 1}$. Since the variance of the Gaussian noise added is $\frac{1}{n_{i,t} + 1}$, and assuming that $n_{i,t}\geq 1$ -- meaning that at least one reward has been realized from arm $i$, this yields a $\sqrt{\frac12}$-GDP guarantee for this $\theta_{i,t+1}$. Since neighboring reward sequences can only differ in one single reward, they will also differ in only one single arm pull, so we need not consider composition across all $N$ arms. Finally, the outcome of $\M_{TS}(\mathcal{F}_t)$ is simply the argmax of the $\theta_i$, which is post-processing on the private output. 

 In Appendix \ref{app.otherprivacy}, we give an alternative proof for the DP guarantees of $\M_{TS}(\mathcal{F}_t)$, by proving that ReportNoisyMax \cite{DR14} with heterogeneous Gaussian noise, rather than the standard Laplace noise with identifical variance, satisfies DP. This result may be of independent interest, but leads to looser overall privacy bounds for Thompson Sampling due to the composition over a large number of rounds.

Next, we apply the composition guarantees of GDP given in Lemma \ref{lem.GDP-comp} to show that the repeated application of $\M_{TS}(\mathcal{F}_t)$ for $T$ rounds -- as in Thompson Sampling -- also satisfies GDP.

\begin{lemma} \label{lem.orig-gdp}
 The Thompson Sampling algorithm with Gaussian priors run for $T$ total timesteps (Algorithm \ref{alg.ts.orig}) satisfies $\sqrt{\frac{1}{2} T} $-GDP. 
\end{lemma}

The proof of Lemma \ref{lem.orig-gdp} follows immediately from the fact that one round of Thompson Sampling is $\sqrt{\frac12}$-GDP by Lemma \ref{lem.tsgdp-onestep}, and then applying GDP composition (Lemma~\ref{lem.GDP-comp}) to get that $T$ rounds of Thompson Sampling is together $\sqrt{\frac{1}{2} T} $-GDP.

Finally, we use Lemma~\ref{lem.gdp-to-eps} to convert the GDP guarantee of Lemma~\ref{lem.orig-gdp} back to the desired $(\epsilon, \delta)$-DP. 

\begin{theorem}\label{thm.tsdp}
    Thompson Sampling with Gaussian priors (Algorithm \ref{alg.ts.orig}) run for $T$ timesteps is  $(\epsilon, \delta(\epsilon) ) $-DP for all $\epsilon \geq 0 $, where $\delta(\epsilon) = \Phi(-\frac{\epsilon}{\sqrt{\frac12 T}} + \frac12 \sqrt{\frac12 T}) - e^{\epsilon} \Phi(-\frac{\epsilon}{\sqrt{\frac12 T}} - \frac12 \sqrt{\frac12 T} )$. 
\end{theorem}

\begin{remark}
Although GDP provides the tightest composition bounds, this comes at the cost of interpretability of the resulting differential privacy parameters, as observed in the statement of Theorem \ref{thm.tsdp}. In Appendix \ref{app.otherprivacy}, we show that the informal Theorem \ref{thm.dpinformal} can also be proved using the composition methods of standard DP (Theorem \ref{thm.tsregulardp}) and a variant known as Renyi DP (Theorem \ref{thm-rdp}). We also show empirically in Appendix \ref{app.privacycompare} that although these other privacy methods provide more interpretable bounds in terms of the dependence on $T$, the privacy guarantees are often orders of magnitude weaker than those attained using GDP. 
\end{remark}

\begin{remark}
We note that Lemmas \ref{lem.tsgdp-onestep} and \ref{lem.orig-gdp} can be tightened, since the factor $1/2$ in these bounds comes from a loose bound on $\frac{1}{n_{i,t}+1}$ using the fact that $n_{i,t}\geq 1$. However, tightening these lemmas would not in general improve the guarantees of Theorem \ref{thm.tsdp}, since the binding term in the expression comes from the worst-case arm, which could indeed have only been pulled once. However, instance specific accounting might lead to improved realized guarantees empirically, as we observe in Section \ref{s.experiments}.
\end{remark}

\section{Improving the Privacy-Regret Trade-off}\label{s.improving}

In this section, we show how the privacy-regret tradeoff of Thompson Sampling can be improved with a simple modification of the algorithm. Concretely, the modified algorithm first pulls each arm $b$ times before beginning the Thompson Sampling procedure. This serves to give the algorithm a ``warm start'' with more accurate prior beliefs on rewards, rather than $\mathcal{N}(0,1)$. It also decreases the sensitivity of the implicit Gaussian Mechanism that computes $\theta_{i,t}$ by ensuring that each $n_{i,t}$ is at least $b$, thus leading to smaller $\epsilon$ values; on the other hand, the algorithm does not improve its decisions during these $bN$ rounds, and may incur maximum loss during these initial rounds.

The second modification is scaling up the variance used to sample $\theta_{i,t}$ by a factor of $c \geq 1$. This serves to improve the $\epsilon$ privacy guarantees of the algorithm since more noise is added, but it also adds higher levels of noise to the algorithm's estimated empirical reward of each arm, thus increasing regret. The complete Modified Thompson Sampling algorithm with both of these changes is presented in Algorithm~\ref{alg.ts.modified}.

\begin{algorithm}
\begin{algorithmic}[1]
\caption{Modified Thompson Sampling with Gaussian priors}\label{alg.ts.modified}
\State \textbf{Input:} number of arms $N$, time horizon $T$, number of pre-pulls $b$, variance multiplier $c\geq 1$
\State Initialize $\hat{\mu}_{i, 0} = 0, \  n_{i,0} = 0$ for each $i = 1, \ldots, N$
\For {$i = 1, 2, \ldots, N$}
\For {$j = 1, 2, \ldots, b$}
\State Play arm $i$ and realize reward $r_{i,j}$ 
\State \textbf{Output} $i$
\State Update $ \hat{\mu}_{i,} = \frac{\hat{\mu}_{i,0}   (n_{i,0}  + 1) + r_{i,j}}{n_{i, 0} + 2}$, and $ n_{i,0} = n_{i,0} + 1 $
\EndFor
\EndFor
\For {$t = 1, 2, \ldots, T-bN$}
    \State For each arm $i = 1, \ldots, N$, sample independently $\theta_{i,t} \sim \N(\hat{\mu}_{i,t-1}, \frac{c}{n_{i,t-1}+1})$
    \State Play arm $a_t : = \arg \max _i \theta_{i,t}$ and realize $r_t$
    \State \textbf{Output} $a_t$
    \State For $i=a_t$, update $ \hat{\mu}_{i, t} = \frac{\hat{\mu}_{i, t-1}  (n_{i,t-1} + 1) + r_t}{n_{i, t-1} + 2}$, and $ n_{i,t} = n_{i,t-1} + 1$
    \State For all $i \neq a_t$, update $ \hat{\mu}_{i, t} =\hat{\mu}_{i, t-1}$, and $ n_{i,t} = n_{i,t-1}$
\EndFor
\end{algorithmic}
\end{algorithm}

We show that by tuning $b$ and $c$, the Modified Thompson Sampling algorithm can achieve both tighter privacy guarantees and lower regret, than the values achieved under the existing settings of $b=0$ and $c=1$ that correspond to standard Thompson Sampling (Algorithm \ref{alg.ts.orig}). The remainder of this section provides analysis of the privacy guarantee and the regret bound of Algorithm~\ref{alg.ts.modified}.

\subsection{Privacy Guarantees}

The privacy analysis of Algorithm \ref{alg.ts.modified} follows a similar structure as that of Algorithm \ref{alg.ts.orig} in Section \ref{s.tsisdp}. We start with Lemma \ref{gdp-ts-b}, which gives the GDP guarantee.

\begin{lemma}\label{gdp-ts-b}
    Modified Thompson Sampling with Gaussian priors and input parameters $(b,c)$ run for $T$ timesteps satisfies $\sqrt{ \frac{1}{c(b+1)} T} $-GDP. 
\end{lemma}

The full proof of Lemma \ref{gdp-ts-b} is presented in Appendix \ref{app.modifiedprivacy}, and we give a proof sketch here.  Similar to the proof of Lemma \ref{lem.tsgdp-onestep} in Section \ref{s.tsisdp}, the proof begins with a privacy analysis of the single step of the mechanism at a fixed time $t$. The changes for this modified algorithm are in the sensitivity of $\muit$ and in the noise that is added. Recall that the GDP parameter of the Gaussian Mechanism is $s_f/\sigma$ when sensitivity of the function is $s_f$ and the added Gaussian noise has variance $\sigma^2$. In Algorithm \ref{alg.ts.modified}, this expression is:
\[     
\left| \tfrac{\frac{1}{n_{i,t} + 1 }}{\sqrt{\frac{c}{n_{i,t} + 1}}}\right|  
    =  \tfrac{1}{\sqrt{ c( n_{i,t} + 1 )} } 
    \leq  \tfrac{1}{\sqrt{c(b+1)}}.\]

This single shot GDP guarantee is then composed across $T$ rounds using Lemma \ref{lem.GDP-comp}, to give $\sqrt{ \frac{1}{c(b+1)} T} $-GDP over $T$ rounds.

Finally, the GDP guarantee of Lemma \ref{gdp-ts-b} is converted to a differential privacy guarantee using Lemma \ref{lem.gdp-to-eps} to yield the final privacy guarantees of Algorithm \ref{alg.ts.modified}, presented in Theorem \ref{thm.mtsprivacy}.

\begin{theorem}\label{thm.mtsprivacy}
    Modified Thompson Sampling with Gaussian priors and input parameters $(b,c)$ run for $T$ timesteps satisfies $(\epsilon, \delta(\epsilon) ) $-DP for all $\epsilon \geq 0 $, where $\delta(\epsilon) = \Phi(-\tfrac{\epsilon}{\sqrt{  \frac{1}{c(b+1)} T}} +  \tfrac12 \sqrt{ \tfrac{1}{c(b+1)} T} ) - e^{\epsilon} \Phi(-\tfrac{\epsilon}{\sqrt{  \frac{1}{c(b+1)} T}} - \tfrac12 \sqrt{  \tfrac{1}{c(b+1)} T} ).$
\end{theorem}

To illustrate the impact of $b$ and $c$ on the privacy guarantees of Theorem \ref{thm.mtsprivacy}, Figure~\ref{fig:impact_b_c} visualizes the tradeoff between $\epsilon$ and $\delta$ under varying $b$ and $c$.  We observe that relative to the values $b=0$ and $c=1$ corresponding to the special case of standard Thompson Sampling, increasing these parameters can lead to substantial improvements in the privacy guarantee. This means that even a small amount of prepulling or increase in the variance of sampling $\theta$ can lead to dramatically tighter privacy guarantees, relative to the standard Thompson Sampling algorithm. Further empirical analysis, including the impact of $b$ and $c$ on regret, is deferred to our experimental results in Section \ref{s.experiments}.

\begin{figure}[ht]
    \centering
    \includegraphics[scale=0.5]{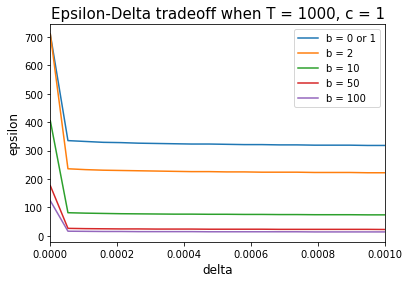}
    \includegraphics[scale=0.5]{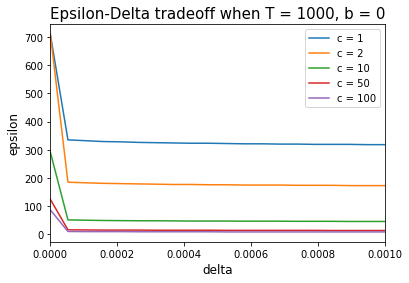}
    \caption{\small{DP parameter $\epsilon$ as a function of $\delta$ when fixing $T=1000$ and varying $b$ (left) and $c$ (right). Note that the per-round GDP parameter is $\sqrt{\frac{1}{c(b+1)}}$, so the role of $c$ and $b$ are nearly symmetric, leading to nearly identical plots on the left and right. Of course, these parameters can also be varied together.} 
}
    \label{fig:impact_b_c}
\end{figure}

\subsection{Regret Guarantees}

Since Algorithm \ref{alg.ts.modified} differs from the standard Thompson Sampling algorithm, new regret analysis is needed. Theorem \ref{thm:regret} gives both problem-dependent and problem-independent regret bounds for Algorithm \ref{alg.ts.modified}, based on both parameters $b$ and $c$. Relative to the guarantees of Theorem \ref{thm.tsregret} from \cite{AG17} for standard Thompson Sampling, we see that regret increases by a factor of $c$. 

\begin{restatable}{theorem}{mainregret}\label{thm:regret} 
Consider the Modified Thompson Sampling with Gaussian priors and input parameters $(b,c)$ run for $T > bN + \frac{4}{\min_i \Delta_i^2}$ timesteps, where $b \geq 0, c\geq 1$. Then the algorithm has expected regret $bN + O(c \sqrt{N(T-bN) \log N})$ (problem-independent), or $bN + \sum_{i=1}^N O(c \frac { \log (T-bN)  } {\Delta_i})$ (problem dependent). 
\end{restatable}

The full proof of Theorem \ref{thm:regret} is deferred to Appendix \ref{app.proofs}. It generalizes the analysis of \cite{AG17} to incorporate the new parameters $b$ and $c$. Without loss of generality, we assume arm 1 is the unique optimal arm. The key idea is to note that $\E[\mathcal{R}(T, \pi)]=\sum_{i=1}^N\Delta_i\E[n_{i,T}]$ and $\E[n_{i,T}]=\sum _{t = 1} ^T \Pr [a_t = i ]$. Therefore in order to control the expected regret it suffices to control the probability that arm $i$ gets play at any time $t$. This can be done by intersecting this probability to events of the form  $\{\muitminus-\mu_i \leq \Delta_i/3\}$ and $\{\thetait -\mu_1\leq -\Delta_i/3\}$, where arm 1 is assumed WLOG to be the optimal arm. Since rewards are bounded in $[0,1]$ and the $\thetait$ are Gaussian, one can tightly  control these probabilities by carefully conditioning on the history using Chernoff-type bounds tailored for this problem.

\section{Experiments}\label{s.experiments}

In this section, we evaluate the empirical performance of the modified Thompson Sampling algorithm under varying $b$ and $c$ parameter values. This both helps illustrate performance of the algorithm in terms of privacy and regret, and it also helps illustrate the role of the additional parameters. Recall that $b=0$ and $c=1$ is the special case corresponding to standard Thompson Sampling.  We vary combinations of $(b,c)$ that achieve the same fixed privacy budget, as measured by the GDP parameter, which we also vary. These experiments can also provide insight for identifying the optimal $(b,c)$ values to minimize regret given a fixed privacy budget. We consider two families of true reward distributions: Bernoulli and exponentially distributed. 
In Appendix~\ref{app.comparison}, we also compare our algorithm against two recent non-TS-based DP algorithms for online bandit problems: DP-SE \cite{SS19} and Anytime-Lazy-UCB \cite{HU21}). All experiments were run on a personal laptop with an M1 Pro chip in around 2 hours. 

\subsection{Bernoulli rewards}

We start with the experimental setting of \cite{HH22}, where $N=5$ arms have Bernoulli rewards with means [0.75, 0.625, 0.5, 0.375, 0.25] respectively.
We let $T=10^5$ and vary the desired $\eta$-GDP privacy parameter to be $1, 2$ and $5$. We vary the $(b,c)$ parameters jointly to ensure that the desired GDP guarantee is obtained. Recall from Lemma \ref{gdp-ts-b} that to ensure $\mu$-GDP for $T=10^5$ rounds, $b$ and $c$ must satisfy $\eta=\sqrt{ \frac{1}{c(b+1)}}10^{2.5}$. 

Note that standard Thompson Sampling with $b=0$ and $c=1$ would yield $10^{2.5} \approx 316$-GDP, which is orders of magnitude higher than the privacy parameters considered here, so performance for these parameter values are not shown in the plots.

\begin{figure}[tbh]
    \centering
\includegraphics[scale=0.4]{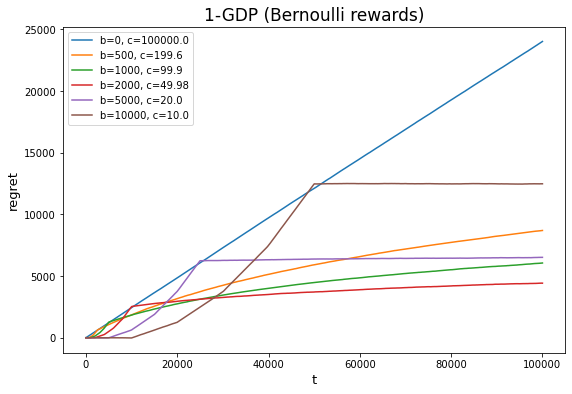}
\includegraphics[scale=0.4]{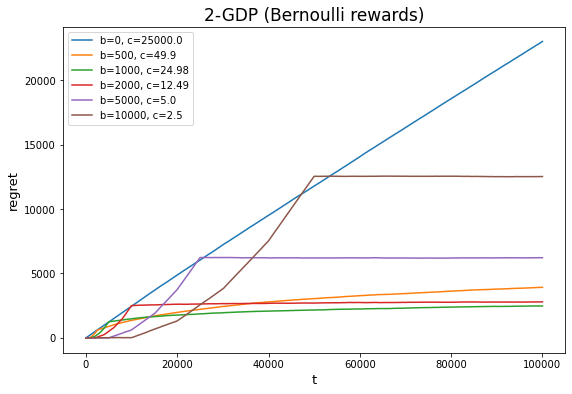}
\includegraphics[scale=0.4]{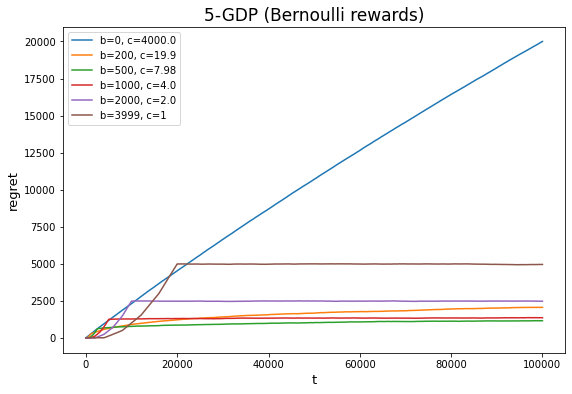}
    \caption{\small{Empirical regret of Algorithm \ref{alg.ts.modified} under varying $(b,c)$ when rewards are Bernoulli distributed.}
    }
    \label{fig:sec5set1}
\end{figure}

Figure \ref{fig:sec5set1} shows the \emph{empirical regret} for each parameter combination over time, defined as,
\begin{equation}\label{eq.empregret}
\E[\mathcal{R}(T, \pi)] = \mustar T - \textstyle\sum_{t=1}^T r_{t}.
\end{equation}
The empirical regret is averaged over 10 runs, for which we already observe convergence of the average regret empirically. 
For each parameter setting, we observe an initial period of high regret, corresponding to the pre-pulling phase; this is more pronounced for larger $b$ values. Afterwards, there is a phase transition to much lower per-round regret, once the algorithm begins the traditional Thompson Sampling phase. Parameter regimes with larger $b$ values perform extremely well in this phase, both because they have a more accurate warm-start from the pre-pulls, and because larger $b$ corresponds to smaller $c$ for the same fixed GDP guarantees, corresponding to lower variance in the sampling step. Smaller $b$ values do not suffer this initial period of loss, but they do incur more per-round regret; at the extreme, $b=0$ has so much noise that its estimates do not converge. 
The convex shape of the regret in the pre-pulling phase is an artifact of our specific setting, where the algorithm pre-pulls the arms with the highest average rewards first. 

We also observe that the lowest regret at time $T$ is achieved by parameter settings with intermediate values of $b$ and $c$, rather than settings that only involve tuning each parameter alone. This suggests that an optimal tuning strategy would increase both $b$ and $c$ together. Finally, we observe that the values of $b$ and $c$ that lead to the lowest empirical regret depend on the privacy budget, meaning that parameter tuning to optimize the privacy-regret tradeoff must take into account the desired privacy level. %The regret for all parameters decreases with weaker privacy guarantees, as expected.

\begin{figure}[htb]
    \centering
    \includegraphics[scale=0.6]{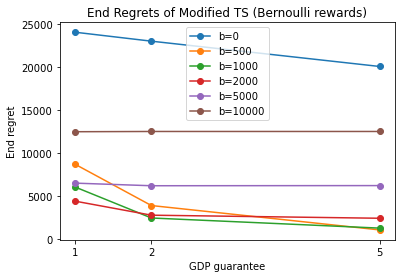}
    \caption{\small{Total regret at $T=100,000$ of the modified Thompson Sampling algorithm under varying privacy guarantees, when rewards are Bernoulli distributed.} }
    \label{fig:newpic_endregrets_diffgdp_bern}
\end{figure}

Figure~\ref{fig:newpic_endregrets_diffgdp_bern} shows impact of the parameters $b$ and $c$ on the privacy-regret trade-off of modified Thompson Sampling (Algorithm~\ref{alg.ts.modified}) by plotting the final empirical regret at time $T=100,000$ against the GDP guarantee for varying values of $b$. For each fixed privacy guarantee and $b$ value, we use the minimum $c$ value required to satisfy the privacy budget. These $c$ values are not reported in the legend because they also vary with the GDP parameter, but they can be read off of the corresponding plots in Figures \ref{fig:sec5set1}.

Figure~\ref{fig:newpic_endregrets_diffgdp_bern} further illustrates that optimal parameter tuning should involve intermediate $b$ and $c$ values. We observe that the highest and lowest values of $b$ correspond to the highest regret, across all privacy levels, and that substantially lower regret is achieved with intermediate values of $b$. We also observe that, as expected, regret is lower with lower levels of privacy protection. This effect is more pronounced with smaller values of $b$. One possible explanation is that when $b$ is large, the primary source of regret is the pre-pulls, so varying the privacy level does not significantly impact the overall regret. However, when $b$ is smaller, varying the privacy level does impact the regret more significantly.

\subsection{Truncated exponential rewards}

Next, we consider rewards sampled from the truncate exponential distribution on [0,1] with varying rates.  We again consider $N=5$ arms respectively with exponential rates [0.1, 1, 2, 5, 10], corresponding to means of approximately [0.492, 0.418, 0.343, 0.193, 0.1].  We again fix $T=10^5$, and vary the $(b,c)$ parameters jointly to achieve desired GDP guarantees of $\eta = 1,2,5$.

Figure \ref{fig:sec5set2} shows the empirical regret as defined in Equation \eqref{eq.empregret} for each parameter combination over time, averaged over 10 runs. We observe qualitatively similar findings as with the Bernoulli rewards: larger $b$ corresponds to high initial regret and then low per-round regret; larger $c$ corresponds to higher per-round regret, with the largest value of $c$ (i.e., $b=0$) having much higher regret; the optimal regret-minimizing parameter regimes involve intermediate values of $b$ and $c$; and the optimal parameter values depend on the desired GDP parameter.

\begin{figure}[tbh]
    \centering
\includegraphics[scale=0.4]{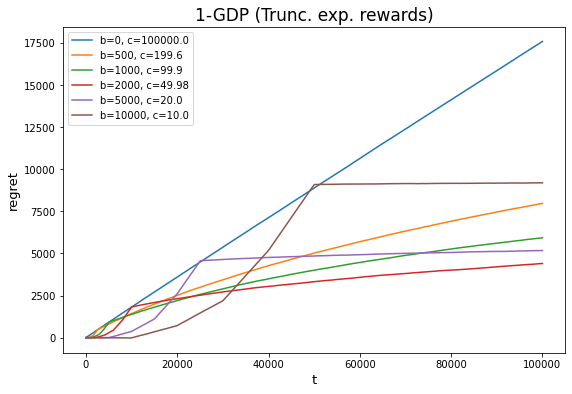}
\includegraphics[scale=0.4]{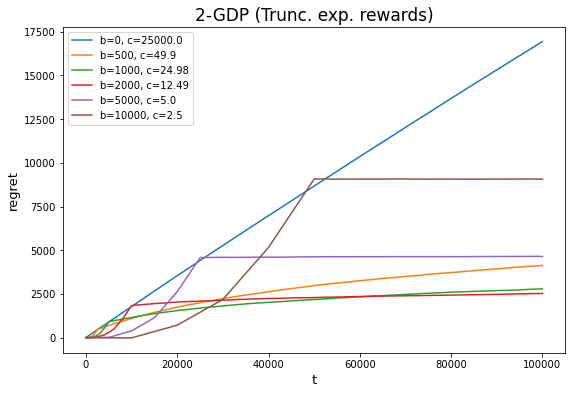}
\includegraphics[scale=0.4]{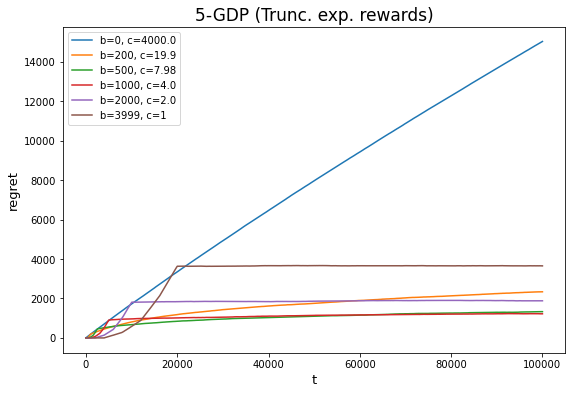}
    \caption{\small{Empirical regret of Algorithm \ref{alg.ts.modified} under varying $(b,c)$ when rewards are generated from a truncated exponential distribution.}
    }
    \label{fig:sec5set2}
\end{figure}

Figure~\ref{fig:newpic_endregrets_diffgdp_te} shows the total regret at time $T=100,000$ for different privacy levels and $b$ values -- with their corresponding $c$ values as indicated in Figure \ref{fig:sec5set2}. We also observe similar findings to the case of Bernoulli rewards: the lowest regret is achieved with intermediate values of $b$, with substantially higher regret under the largest and smallest values of $b$. For smaller $b$ values, we observe that increasing the privacy parameter (i.e., weaker privacy) leads to lower regret, whereas this effect is less noticeable for larger values of $b$.

\begin{figure}[ht]
    \centering
    \includegraphics[scale=0.6]{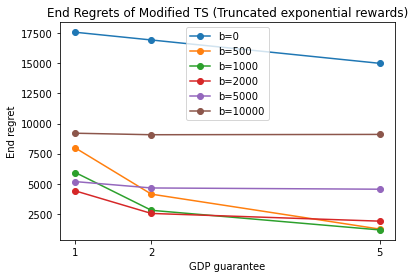}
    \caption{\small{Total regret at $T=100,000$ of the modified Thompson Sampling algorithm under varying privacy guarantees, when rewards are generated from truncated exponential distributions.}}
    \label{fig:newpic_endregrets_diffgdp_te}
\end{figure}

\section{Conclusion}
In this work, we analyze the privacy guarantees of the Thompson Sampling algorithm \cite{T1933,russoetal2018} with Gaussian priors, which is commonly used for learning with bandit feedback. We show that the original Thompson Sampling algorithm satisfies differential privacy without any modifications, leveraging structural similarities between the algorithm's sampling procedure and existing DP tools, namely the Gaussian Mechanism. Importantly, this result means that there is \emph{no} loss in performance from adding privacy, and known regret bounds \cite{AG17} still hold for the private algorithm.

Additionally, we show that two small modifications to the algorithm -- namely pre-pulling each arm $b$ times and scaling up the variance of sampling noise by a factor of $c$ -- enables tunable privacy guarantees. The resulting privacy and regret guarantees depend on the values of the new parameters $b$ and $c$, which can be tuned to substantially improve the privacy guarantee at only a small increase in regret. We demonstrate our theoretical results empirically on two different reward distributions, and show substantial improvements in regret for a fixed privacy guarantee by properly tuning the parameters $b$ and $c$.

\bibliography{ref}
\bibliographystyle{alpha}

\newpage 

\appendix

\section{Omitted Privacy Proofs}\label{app.omitted}

\subsection{Proof of Lemma \ref{lem.tsgdp-onestep}}\label{app.tsonestep}

\begin{lemma2}
The mechanism $\M_{TS}(\mathcal{F}_t)$ satisfies $\sqrt{\frac12}$-GDP with respect to observed rewards.
\end{lemma2}

\begin{proof}

We first note that Thompson Sampling itself is inherently an adaptive algorithm whose output at any time step $t$ depends on its previous outputs at time steps $1,2,\ldots, t-1$. In order to prove the privacy guarantee of Algorithm~\ref{alg.ts.orig} composed for $T$ time steps, it suffices to show that at every time step, conditioning on the outputs in the previous time steps, the algorithm is $\sqrt{1/2}$-GDP. 
To prove Lemma \ref{lem.tsgdp-onestep}, consider two neighboring histories $\mathcal{F}_t$ and $\mathcal{F}'_t$ of length $t$ that differ only in one observed reward.
Consider the vector-valued $N$-dimensional query $\hat{\mu}(\mathcal{F}_t) = (\hat{\mu}_{1}(\F_t), \ldots, \hat{\mu}_{N}(\F_t))$ that computes the empirical mean of observed rewards from each arm, given reward history $\F_t$:
\[\hat{\mu}_{i}(\F_t) = \frac{1}{n_{i,t}+1}\sum _{\tau = 1}^t 1_{a_{\tau}=i} r_{\tau}. \]
Then for each $i \in [N]$, $\theta_{i,t+1}$ can be expressed as $\hat{\mu}_{i}(\F_t)$ plus an independent Gassian noise term sampled from $\N (0, \frac{1}{n_{i,t} + 1})$.

We then analyze the privacy guarantee of releasing $\{ \theta_{i,t+1} \} _{i \in [N]}$,  which is sufficient to determine the arm pulled at time $t+1$, as simply the argmax of all $\theta_{i,t+1}$. Let $j$ denote this arm. Since $\theta_{j,t+1}$ is the only value that would depend on the reward observed at time $t$, then it would be the only value that differs between neighboring databases of  reward histories $\F_t$ and $\F'_t$. Therefore the (single shot) $N$-dimensional empirical mean query has sensitivity $s = \max | 
 \hat{\mu} _{j}(\F_t) - \hat{\mu}_{j}(\F'_t)| \leq \frac{1}{n_{j,t} + 1}$. The standard deviation of noise added is $\sqrt{\frac{1}{n_{j,t}+1}}$, so the GDP parameter is at most:
 \begin{align}
    \left| \frac{\frac{1}{n_{j,t} + 1}}{\sqrt{\frac{1}{n_{j,t} + 1}}}\right|  = \frac{1}{\sqrt{n_{j,t} + 1}} \leq \frac{1}{\sqrt{2}}.
    \label{Single-step-gdp}
\end{align}

The final inequality in Equation \eqref{Single-step-gdp} comes from the assumption that $n_{j,t} \geq 1$, which is required to avoid the degenerate case of empty database $\hat{\mu}_{j,t}$ used in the Gaussian mechanism.
Note that if this were not the case, then $\theta_{j,t+1} \sim \N( 0,1 )$, and there would be no privacy loss associated with arm $j$ because there would be no data to protect.

\end{proof}

\subsection{Proof of Lemma \ref{gdp-ts-b}}\label{app.modifiedprivacy}

\begin{lemma4}
Modified Thompson Sampling with Gaussian priors and input parameters $(b,c)$ run for $T$ timesteps satisfies $\sqrt{ \frac{1}{c(b+1)} T} $-GDP. 
\end{lemma4}

\begin{proof}
We first show the per-round GDP guarantee of Algorithm~\ref{alg.ts.modified} with respect to the observed rewards is $\sqrt{\frac{1}{(b+1)c}}$, and then compose across rounds using Lemma~\ref{lem.GDP-comp} to reach the final result. To prove the per-round privacy guarantee, consider two neighboring histories $\mathcal{F}_t$ and $\mathcal{F}'_t$ of length $t$ that differ only in one observed reward. 
Consider the vector-valued $N$-dimensional query $\hat{\mu}(\mathcal{F}_t) = (\hat{\mu}_{1}(\F_t), \ldots, \hat{\mu}_{N}(\F_t))$ that computes the empirical mean of observed rewards from each arm, given reward history $\F_t$:
\[\hat{\mu}_{i}(\F_t) = \frac{1}{n_{i,t}+1}\sum _{\tau = 1}^t 1_{a_{\tau}=i} r_{\tau}. \]
For each $i \in [N]$, $\theta_{i,t+1}$ can be expressed as $\hat{\mu}_{i}(\F_t)$ plus an independent Gassian noise term sampled from $\N (0, \frac{c}{n_{i,t} + 1})$.

Following the same argument as in the proof of Lemma~\ref{lem.tsgdp-onestep}, we analyze the privacy guarantee of releasing $\{ \theta_{i,t+1} \} _{i \in [N]}$, which is sufficient to determine the arm pulled at time $t+1$ as $\arg\max_i \theta_{i,t+1}$. Denote this arm as $j$. As in the proof of Lemma \ref{lem.tsgdp-onestep}, $\theta_{j,t+1}$ is the only value that would depend on the reward observed at time $t$, so it would be the only value that differs between neighboring histories $\F_t$ and $\F'_t$. Therefore the (single shot) $N$-dimensional empirical mean query is essentially a histogram query with sensitivity $s = \max | 
 \hat{\mu} _{j}(\F_t) - \hat{\mu}_{j}(\F'_t)| \leq \frac{1}{n_{j,t} + 1}$. 
 
To calculate the GDP parameter, with a variance multiplier $c$, the standard deviation of the Gaussian noise added is now $\sqrt{\frac{c}{n_{j,t}+1}}$. Additionally, we have a new lower bound for $n_{j,t}$, of $b$, due to the $b$ pre-pulls of each arm. Then the GDP parameter is at most:
\begin{align*}
    \left| \frac{\frac{1}{n_{j,t} + 1}}{\sqrt{\frac{c}{n_{j,t} + 1}}}\right|  = \frac{1}{\sqrt{c (n_{j,t} + 1 )}} \leq \frac{1}{\sqrt{c(b+1)}}.
\end{align*}

Then, using Lemma~\ref{lem.GDP-comp} which composes the single-step privacy guarantee adaptively, the GDP parameter over $T$ timesteps is 
\[\sqrt{ \sum _{t=1}^T \left(\frac{1}{\sqrt{c(b+1) } }
\right)^2 } = \sqrt{  \frac{1}{c(b+1)} T}.\]
\end{proof}

%%%%% %%%%% %%%%% %%%%% %%%%% %%%%%
%%%%%% Appendix B %%%%%
%%%%% %%%%% %%%%% %%%%% %%%%% %%%%%
\section{Omitted Regret Proofs}\label{app.proofs}

We adapt the arguments of \cite{AG17} to our modified Thompson sampling algorithm. We first provide the main argument of the proof in Section \ref{app:regret_proof}, followed by proofs of three auxiliary lemmas in Section \ref{app:regret_lemmata}.

\subsection{Proof of Theorem \ref{thm:regret}}
\label{app:regret_proof}

\mainregret*

Recall the notation of Algorithm \ref{alg.ts.modified}, that $a_t$ is the arm played and outputted at time $t$, $n_{i,t}$ is the number of times arm $i$ is pulled after $t$ time steps, and $\muit = \frac{1}{n_{i, t} + 1}(\sum _{j=1}^b r_{i,j} + \sum_{\tau = 1: a_\tau = i} ^{t} r_t)$ is the empirical mean of rewards for arm $i$ after $t$ time steps. At time step $t$, the sample $\theta_{i,t}$ is drawn from the posterior distribution from observations in the previous $t-1$ timesteps, i.e. $\theta_{i,t} \sim \N (\hat{\mu}_{i,t-1}, \frac{1}{n_{i,t-1} + 1})$. We emphasize that at time step $t$, the sample  $\theta_{i,t} \sim \N (\hat{\mu}_{i,t-1}, \frac{1}{n_{i,t-1} + 1})$ is drawn from a distribution updated by data up to time $t-1$. That is, the decision $a_t$ only depends on $\hat{\mu}_{i,t-1}$ and $\frac{1}{n_{i,t-1} + 1}$.

Recall that the expected problem-dependent regret (Equation \eqref{eq:dependentregret}) can be written as: 
\[ \E[\mathcal{R}(T, \pi )]=\sum_{i=1}^N\Delta_i\E[n_{i,T}].\]

We will therefore need to bound the expected number of suboptimal plays of each arm $i$ during $T$ time steps. Without loss of generality, we assume arm 1 is the optimal arm, and define $\Delta_i = \mu_1 - \mu_i$ to be the suboptimality gap of arm $i$; if there are multiple optimal arms, this will only improve regret.

For each $i$, define $x_i = \mu_i + \Delta_i / 3$ and $y_i = \mu_1 - \Delta_i / 3$. These will serve as two ``mid-points" between $\mu_i$ and $\mu_1$.  Thus for $i\neq 1$, we have $\mu_i < x_i < y_i < \mu_1$, and for $i = 1$, we have $\mu_i = \mu_1 = x_1 = y_1$.  
We will also define the events $\Eimu=\{\muitminus \leq x_i\}$, 
and $\Eitheta=\{\thetait \leq y_i\}$. 
Note that when the number of observed rewards of arm $i$ increases,  empirical mean $\hat{\mu}_{i,t-1}$ convergences to $\mu_i$, making the probability of event $\Eimu$  tend to $1$  and that of $\Eitheta$  tend to $0$.  We denote the respective complements of these events by $\overline{E_i^{\mu} (t)}$ and $\overline{E_i^{\theta} (t)}$. Finally, we define $\pit := \Pr[\theta_{1,t} > y_i | \F_{t-1}]$. Note that this is a random variable that depends on $ \F_{t-1}$, and has a fixed value if  $\F_{t-1}$ is instantiated to be a particular history $F_{t-1}$.

Given $b$ prepulls per arm and $T>bN$, this notation and the law of total probability leads to the following identity on the number of arm pulls. 
\begin{align}
   \nonumber \E [\nitcap ] & = b+\sum _{t = 1} ^{T-bN} \Pr \left[a_t = i \right] \\
 \label{eq:thm_regret_eq1}   & = b+ \sum _{t = 1} ^{T-bN} \Pr \left[a_t = i, \overline{E_i ^{\mu} (t)} \right] 
    + \sum _{t = 1} ^{T-bN} \Pr \left[a_t = i, E_i ^{\mu} (t), \overline{E_i^{\theta} (t)} \right] +\sum _{t = 1} ^{T-bN} \Pr \left[a_t = i, E_i ^{\mu} (t), E_i^{\theta} (t) \right].
\end{align}
We will upper bound each of these three terms separately.
Lemmas \ref{lem.2.15} and \ref{lem.2.16} bound the first two terms in Equation \eqref{eq:thm_regret_eq1}, and are both proven in Section \ref{app:regret_lemmata}.

\begin{restatable}{lemma}{firstterm}\label{lem.2.15}
%\begin{lemma}\label{lem.2.15}
$\sum _{t = 1} ^{T-bN} \Pr \left[a_t = i,  \overline{E_i^{\mu} (t)} \right] \leq\frac{9}{\Delta_i^2}e^{-2b\Delta_i^2/9}.$ 
%\end{lemma}
\end{restatable}

\begin{restatable}{lemma}{secondterm}\label{lem.2.16}
Let $T\geq bN+\frac{1}{\Delta_i^2}e^{\frac{1}{4\pi}}$. Then, $\sum _{t = 1} ^{T-bN} \Pr \left[a_t = i, E_i ^{\mu} (t), \overline{E_i^{\theta} (t)} \right]   \leq \max\{0, \frac{18c\log ((T-bN) \Delta_i ^2)}{\Delta_i^2 } - b \} +  \frac{1}{\Delta_i^2}$. 
\end{restatable}

In order to bound the third term in Equation \eqref{eq:thm_regret_eq1}, we will need Lemma \ref{lem.2.8}.

\begin{lemma}[\cite{AG17}] \label{lem.2.8}
For all $t$, $i \neq 1$ and histories $F_{t-1}$ we have 
\begin{align*}
    \Pr \left[a_t = i, \Eimu, \Eitheta | \F_{t-1}= F_{t-1}  \right] \leq \frac{1 - \pit}{\pit} \Pr \left[a_t = 1, \Eimu, \Eitheta | \F_{t-1}=F_{t-1} 
    \right].
\end{align*}
\end{lemma}

Taking conditional expectations with respect to $\mathcal{F}_{t-1}$, followed by Lemma~\ref{lem.2.8}  we see that
\begin{align}
\sum _{t = 1} ^{T-bN} \Pr \left[a_t = i, E_i ^{\mu} (t), E_i^{\theta} (t) \right] 
&= \sum _{t = 1} ^{T-bN} \E \left[ \Pr (a_t = i, E_i ^{\mu} (t), E_i^{\theta} (t) | \F_{t-1} ) \right] \nonumber \\
& \leq  \sum _{t = 1} ^{T-bN} \E \left[  \frac{1-\pit}{\pit} \Pr (a_t = 1, E_i ^{\mu} (t), E_i^{\theta} (t) | \F_{t-1} )\right]  \hspace{0.5cm} \nonumber \\
& =  \sum _{t = 1} ^{T-bN} \E \left[\E \left[ \frac{1-\pit}{\pit} \mathbbm{1} \big(a_t = 1, E_i ^{\mu} (t), E_i^{\theta} (t)  \big) \Big| \F_{t-1} \right]   \right]\nonumber \\
& =  \sum _{t = 1} ^{T-bN} \E \left[\frac{1-\pit}{\pit} \mathbbm{1} (a_t = 1, E_i ^{\mu} (t), E_i^{\theta} (t)  ) \right].  \label{step.1stterm.1}
\end{align}
The second step is an application of Lemma~\ref{lem.2.8}, and the third step uses the fact that $\pit$ is fixed given $F_{t-1}$ and the outer expectation is taken over all possible histories $\F_{t-1}$.

Now, let $\tau_k$ be the time at which arm 1 is played for the $k$-th time, not counting the pre-pull stage, so that $n_{1,\tau_k} = b+k$. 
Note that $\pit =  \Pr[\theta_{1,t} > y_i | \F_{t-1}]$ changes only if the distribution of $\theta_{1,t}$ changes. Thus, $\pit$ is the fixed for all $t \in \{ \tau_k + 1, \ldots, \tau_{k+1} \}$ for every $k$. Then,
\begin{align}
\sum _{t = 1} ^{T-bN} \E \left[\frac{1-\pit}{\pit} \mathbbm{1} (a_t = 1, E_i ^{\mu} (t), E_i^{\theta} (t)  ) \right]  & = \sum _{k = 0} ^{T-bN - 1} \E \left[ \frac{1 - p_{i, \tau_k + 1}}{p_{i, \tau_k + 1}}  \sum _{t=\tau_k+1} ^ {\tau_{k+1}}\mathbbm{1} \left(a_t = 1, \Eimu, \Eitheta \right)\right] \nonumber \\
& \leq \sum _{k = 0} ^{T - bN-1} \E \left[ \frac{1 - p_{i, \tau_k + 1}}{p_{i, \tau_k + 1}} \right] \nonumber\\
&= \sum _{k = 0} ^{T-bN -1} \E \left[ \frac{1 }{p_{i, \tau_k + 1}}-1\right].   \label{step.Ebreakdown} 
\end{align}
The first step re-indexes from $t$ to $\tau_k$, and the second step is because we know arm 1 will only be played once for $t \in \{ \tau_k + 1, \ldots, \tau_{k+1} \}$.

To continue bounding $ \E [  \frac{1}{p_{i, \tau_{k} + 1}} ] - 1$, we require Lemma \ref{lem.2.13}, which is also proved in Section \ref{app:regret_lemmata}.

\begin{restatable}{lemma}{thirdtermlem}\label{lem.2.13}
%\begin{lemma} \label{lem.2.13}
Let $\tau_k$ be the first time when arm 1 is played for the  $k$-th time excluding the pre-pulls, i.e. $n_{1,\tau_k} = b + k$ and $n_{1,t} < b+ k$ for $t < \tau_k$.  
Then for  $T\geq bN+\frac{4}{\Delta_i^2}$, 
\[
    \E \left[  \frac{1}{p_{i, \tau_{k} + 1}}  \right] -1\leq
\begin{cases}
   71 & \text{for all } k,\\
    \frac{4}{(T-bN)\Delta_i^2}      & \text{for } k \geq \max\{1,\frac{72c}{\Delta_i^2}\log ((T-bN)\Delta_i^2) - (b+1) \}.
\end{cases}
\]
%\end{lemma}
\end{restatable}

For ease of notation, define $L= \lceil  \frac{72c}{\Delta_i^2}\log ((T-bN)\Delta_i^2) - (b+1)\rceil$. Combining \eqref{step.1stterm.1}, \eqref{step.Ebreakdown} and Lemma~\ref{lem.2.13} we finally obtain the following upper bound on the third term in Equation \eqref{eq:thm_regret_eq1}:

\begin{align}
\label{eq:T3}
     \sum _{t = 1} ^{T-bN} \Pr \left[a_t = i, E_i ^{\mu} (t), E_i^{\theta} (t) \right] \leq  & \sum _{k = 0} ^{T-bN -1} \E \left[ \frac{1 }{p_{i, \tau_k + 1}}-1\right] \nonumber \\ 
       \leq & \ 71 L{\mathbbm{1}(L>0)} + \sum_{k = L} ^{T-bN-1} \frac{4}{(T-bN)\Delta_i^2} \nonumber \\
      \leq &  \ 71 L{\mathbbm{1}(L>0)}
     +\frac{4}{\Delta_i^2}.
\end{align}
The first step combines Equations \eqref{step.1stterm.1} and \eqref{step.Ebreakdown}, the second step applies Lemma \ref{lem.2.13}, treating separately the terms in the sum with $k$ larger or smaller than $L$, and conditioning on the case that $L>0$, otherwise the first term in this expression will be upper bounded by 0. The third step upper bounds the final term, since $\frac{T-bN-1-L}{T-bN}<1$.

\paragraph{Putting everything together.} Finally, plugging in the bounds of Lemma \ref{lem.2.15}, Lemma \ref{lem.2.16}, and Equation \eqref{eq:T3} into Equation \eqref{eq:thm_regret_eq1} gives the following bound on the expected number of times each arm $i$ is pulled in a stream of length $T$. Recall that $L= \lceil  \frac{72c}{\Delta_i^2}\log ((T-bN)\Delta_i^2) - (b+1)\rceil$. Then, 
\begin{align}
 \nonumber \E  [n_{i,T}] & \leq \ b+\frac{9}{\Delta_i^2}e^{-2b\Delta_i^2/9} + \max \{ 0, \frac{18c\log ((T-bN) \Delta_i ^2)}{\Delta_i^2 } - b\} + \frac{1}{\Delta_i^2}+71 L \mathbbm{1}(L>0) +\frac{4}{\Delta_i^2} \\
\nonumber \leq \  & b + \frac{9}{\Delta_i^2}  + \frac{18c\log ((T-bN) \Delta_i ^2)}{\Delta_i^2} + \frac{1}{\Delta_i^2} +  71 \cdot \frac{72c}{\Delta_i^2}\log ((T-bN)\Delta_i^2)  + \frac{4}{\Delta_i^2}. \\
 \nonumber = \ &   b + \frac{14}{\Delta_i^2} + \frac{18c\log ((T-bN) \Delta_i ^2)}{\Delta_i^2}  +  \frac{5112 c}{\Delta_i^2}\log ((T-bN)\Delta_i^2)  \\
 = \ &  b + \frac{14}{\Delta_i^2} +   \frac{5130c}{\Delta_i^2}\log ((T-bN)\Delta_i^2) \label{eqn:upperbound_enit}
\end{align}
The first inequality is plugging in the bounds of Lemma \ref{lem.2.15}, Lemma \ref{lem.2.16}, and Equation \eqref{eq:T3} into Equation \eqref{eq:thm_regret_eq1}. The second inequality holds because $\frac{18c\log ((T-bN) \Delta_i ^2)}{\Delta_i^2 }  > 0$ (because $T-bN > e^{1/4\pi} / \Delta_i^2$ by assumption) is an upper bound for $\max \{ 0, \frac{18c\log ((T-bN) \Delta_i ^2)}{\Delta_i^2 } - b\}$, and $\frac{72c}{\Delta_i^2}\log ((T-bN)\Delta_i^2) > 0$ is an upper bound for $L$. The third and fourth steps combine terms. 

Note that in Equation~(\ref{eqn:upperbound_enit}), the term $b$ comes from the pre-pulling stage which we will count separately from the Thompson Sampling stage. We focus on the Thompson Sampling stage for now, and upper bound the regret from the pre-pulling stage, which is at most $bN$ later. Define $ \Tilde{\E} [n_{i,T}]$ to be the expected number of pulls of arm $i$ excluding the $b$ pre-pulls, from $t = 1$  to $T-bN$, in the Thompson Sampling stage (excluding the pre-pulls). Then,
\begin{align*}
\Tilde{\E} [n_{i,T}] \leq    \frac{14}{\Delta_i^2} +   \frac{5130c}{\Delta_i^2}\log ((T-bN)\Delta_i^2).
\end{align*}

To obtain an upper bound on the expected regret due to arm $i$ in the $T-bN$ timesteps of the Thompson Sampling stage, we multiply the above expression by $\Delta_i$. 
\begin{align}
    \Delta_i  \Tilde{\E} [n_{i,T}] \leq \frac{14}{\Delta_i} +   \frac{5130c}{\Delta_i}\log ((T-bN)\Delta_i^2)  \label{eq:per_arm_regret}
\end{align}

 From Equation \eqref{eq:per_arm_regret}, adding up the expected regret over the $N-1$ suboptimal arms and adding back the $bN$ maximum possible regret from the pre-pulling phase, we obtain the desired problem-dependent asymptotic bound  $bN + \sum_{i=1}^N O( c \frac { \log ((T-bN)  \Delta_i^2 ) } {\Delta_i})$.

Moving to the problem-independent bound, we note that the first term in Equation \eqref{eq:per_arm_regret} is decreasing in $\Delta_i$ and the second term will also be decreasing for large enough $\Delta_i$. More precisely, define $f(\Delta_i) = \frac{\log \left( (T-bN)\Delta_i^2 \right)}{\Delta_i}$, so that $f'(\Delta_i) = \frac{2 - \log ((T-bN)\Delta_i^2) }{\Delta_i^2}$. We see that $f'(\Delta_i) < 0$ (i.e., the second term is decreasing in $\Delta_i$) if $\Delta_i \geq \frac{e}{\sqrt{T-bN}}$. 
Therefore, if we consider those arms with $\Delta_i \geq \frac{e \sqrt{N \log N}}{\sqrt{T-bN}}$, the total regret these arms incur in the Thompson Sampling stage would be bounded by:
\begin{align*}
\nonumber \sum_{i=1}^N \left\{\frac{14}{\Delta_i}  + 5130c \cdot \frac{\log \left( (T-bN)\Delta_i^2\right)}{\Delta_i} \right\}  & \leq \sum_{i=1}^N \left\{\left(\frac{14\sqrt{T-bN}}{e \sqrt{N \log N}} - \right) + 5130c \cdot \sqrt{T-bN} \cdot \frac{\log \left( e^2N\log N\right)}{e \sqrt{N \log N}} \right\}  \\
\nonumber & = \frac{14\sqrt{N(T-bN)}}{e \sqrt{\log N}}  + 5130c \cdot \sqrt{T-bN} \frac{\sqrt{N}\log \left( e^2N\log N\right)}{e \sqrt{\log N}}\\
& = O\left(c  \sqrt{N(T-bN)\log N}\right).
\end{align*}
The first inequality above comes from plugging in $\Delta_i = \frac{e \sqrt{N \log N}}{\sqrt{T-bN}}$, which gives an upper bound on the regret, which is decreasing with respect to $\Delta_i$. The second step replaces the summation with a factor of $N$.

For every arm with $\Delta_i \leq \frac{e \sqrt{N \log N}}{\sqrt{T-bN}}$, the total regret due to all of these arms in $T-bN$ time steps is bounded by $(T-bN)\Delta_i  \leq   (T-bN) \frac{e \sqrt{N \log N}}{\sqrt{T-bN}} \leq  e \sqrt{N(T-bN) \log N} =  O(\sqrt{N(T-bN) \log N) }$, because $bN \leq T$. Therefore, if we add up the regrets due to all arms (those with $\Delta_i \geq \frac{e \sqrt{N \log N}}{\sqrt{T-bN}} $ and those with $\Delta_i \leq \frac{e \sqrt{N \log N}}{\sqrt{T-bN}}$), in the Thompson Sampling stage, we get that the total regret is $O\left(c  \sqrt{N(T-bN)\log N}\right) + O\left( \sqrt{N(T-bN)\log N}\right) = O\left(c  \sqrt{N(T-bN)\log N}\right).$
Adding the regret from the pre-pulling stage -- which is at most $bN$ -- and the regret from the Thompson Sampling stage, we conclude that the total regret in all $T$ timesteps is bounded by $ bN + O( c\sqrt{N(T-bN) \log N) }$.

\subsection{Proofs of Auxiliary Lemmas}
\label{app:regret_lemmata}

Lemmas  \ref{lem.2.15}, Lemma \ref{lem.2.16}, and \ref{lem.2.13} can be viewed as extended and refined versions of  Lemmas 2.15, 2.16, and 2.13 in \cite{AG17} respectively.

\firstterm*

\begin{proof}
Recall that $\overline{\Eimu}=\{\muitminus > x_i\}$. Let $\tau_{i,k}$ denote the time we pull arm $i$ for the $k$-th time, excluding the pre-pulling stage. 
Note that  $t\leq \tau_{i,t}$ for all $t\in \mathbb{N}$, and that for $k>n_{i,t}$ it holds that $\tau_{i,k}>t$. Then, 
    \begin{align*}
        \sum _{t=1}^{T-bN} \Pr [a_t = i, \overline{\Eimu} ] \leq \sum _{k=1}^{T-bN} \Pr [\overline{E_i ^{\mu} (\tau_{i, k})}],
    \end{align*}
because at a timestep $t$ where the pulled arm is not $i$, then $\Pr[a_t=i]=0$. Thus the only relevant terms in the sum are the ones at timesteps $\{ t = \tau_{i,k} \} $ for $k = 1, 2, \ldots, T-bN$.

At time $\tau_{i, k}$, the empirical mean $\hat{\mu}_{i, \tau_{i, k}-1}$ used by the algorithm to make the decision is upper bounded by the average of the outcomes of $(b+k-1)$ i.i.d.~plays of arm $i$. We will use Hoeffding's inequality (Lemma \ref{lem:Hoeffding}) to obtain high probability bounds of  the deviations between these empirical means and their true means.
\begin{lemma} [Hoeffding's inequality] 
\label{lem:Hoeffding}
Let $X_1, \ldots, X_n \in [0,1]$ be i.i.d.~and $\E[X_i] = \mu, \forall i$. Let $S_n = X_1 + \ldots + X_n$. Then for all $a \geq 0$, 
\[\Pr [S_n \geq n\mu + a] \leq e^{-2a^2  / n} \quad \text{and} \quad \Pr [S_n \leq n\mu - a] \leq e^{-2a^2  / n}.\]
\end{lemma}

Using the definition $\overline{E_i ^{\mu} (\tau_{i, k})}=\{\hat\mu_{i,\tau_{i,k} - 1} > \mu_i + \Delta_i/3\}$, followed by  Hoeffding's Inequality (Lemma \ref{lem:Hoeffding}),
\begin{align}
\notag
\sum _{k=1}^{T-bN} \Pr \left[\overline{E_i ^{\mu} (\tau_{i, k})}\right] &=  \sum_{k=1}^{T-bN}\Pr\left[ \hat\mu_{i,\tau_{i,k} - 1}-\mu_i> \frac{\Delta_i}{3} \right] \\
& \notag\leq \sum _{k=1}^{T-bN} e^{-\frac{2(k+b-1)\Delta_i^2}{9}} \\
& \notag= \sum _{k=0}^{T-bN-1} e^{-\frac{2(k+b)\Delta_i^2}{9}} \\
& \notag =e^{-\frac{2b\Delta_i^2}{9}} \sum _{k=0}^{T-bN-1} (e^{-\frac{2\Delta_i^2}{9}})^k \\
&= \frac{e^{-2b\Delta_i^2/9}(1-e^{-2(T-bN) \Delta_i^2/9})}{1-e^{-2\Delta_i^2/9}}
\label{eq:lem9proof}
\end{align}
 The last equalities follow by simple manipulations and the computation of a geometric sum. 
Finally, noting that $\frac{1-e^{-xt}}{1-e^{-x}}\leq \frac{1}{1-e^{-x}}\leq \frac{2}{x}$ for $t\geq 0$ and $x\in(0, 2/9]$, we see that \eqref{eq:lem9proof} implies
\begin{equation*}
 \sum _{k=1}^{T-bN} \Pr \left[\overline{E_i ^{\mu} (\tau_{i, k})}\right]   \leq\frac{9}{\Delta_i^2}e^{-2b\Delta_i^2/9}.
\end{equation*}

\end{proof}

\secondterm*

\begin{proof} 
First define $M = \frac{18c\log ((T-bN) \Delta_i ^2)}{\Delta_i^2 }$. Then by the law of total probability,
\begin{align}
  \nonumber  \sum _{t = 1} ^{T-bN} \Pr \left[a_t = i, E_i ^{\mu} (t), \overline{E_i^{\theta} (t)} \right] & = \sum _{t = 1} ^{T-bN} \Pr \left[a_t = i, \nitminus + b \leq  M, E_i ^{\mu} (t), \overline{E_i^{\theta} (t)} \right] \\ 
 \label{eq:lem2.16_eq1}  & \quad\quad  + \sum _{t = 1} ^{T-bN} \Pr \left[a_t = i, \nitminus + b > M, E_i ^{\mu} (t), \overline{E_i^{\theta} (t)} \right]
\end{align}
We can bound the first term in Equation \eqref{eq:lem2.16_eq1} by removing the conditioning, as follows, 
\begin{align}
\nonumber  \sum _{t = 1} ^{T-bN} \Pr \left[a_t = i, \nitminus + b \leq M, E_i ^{\mu} (t), \overline{E_i^{\theta} (t)} \right] & \leq \E \left[ \sum _{t=1}^{T-bN} \mathbbm{1} \left( a_t = i, \nitminus + b \leq  M \right)\right] \\
 \label{eq:lem2.16_eq2}     & \leq  \max\{0, M - b\} 
\end{align}
where the second step holds because for any arm $i$, the timesteps satisfying $a_t = i, \nitminus + b \leq M$ are those in which the arm was pulled at most $M - b$ times.

To bound the second term in Equation \eqref{eq:lem2.16_eq1}, we show that if $\nitminus$ is large and the event $\Eimu$ is satisfied, then the probability that the
event $\Eitheta$ is violated is small. Recall that $\Eitheta$ is the event that $\thetait \leq y_i$.  
\begin{align}
  \nonumber  & \sum _{t = 1} ^{T-bN} \Pr \left[a_t = i, \nitminus + b > M, E_i ^{\mu} (t), \overline{E_i^{\theta} (t)} \right] \\
  \nonumber  &\leq  \   \sum_{t=1}^{T-bN} \Pr \left[ a_t = i, \overline{E_i^{\theta} (t)} \bigg | \nitminus + b > M, \ \Eimu \right]\\
  \nonumber  & \leq  \ \E \left[ \sum_{t=1}^{T-bN} \Pr \left( a_t = i, \overline{E_i^{\theta} (t)} \bigg | \nitminus + b > M, \ \Eimu, \ \F_{t-1} \right)\right] \\
\label{eq:lem2.16_eq3}   & \leq   \ \E \left[ \sum _{t=1}^{T-bN} \Pr \left( \thetait > y_i \bigg | \nitminus + b > M, \ \muitminus \leq x_i, \ \F_{t-1} \right) \right],
\end{align}
where the first step comes from conditioning on additional events, the second step takes the expectation over all histories $\F_{t-1}$, and the third step uses the fact that $\overline{\Eitheta}$ is the event that $\thetait >  y_i$ and $\Eimu$ is the event that $\muitminus \leq x_i$.

Next we will upper bound the probabilities inside the expectation in Equation \eqref{eq:lem2.16_eq3}. Using the fact that $\thetait \sim \N(\hat{\mu}_{i,t-1}, \frac{c}{n_{i,t-1} + b + 1})$, 
\begin{align}
    \nonumber & \E \left[ \Pr \left( \thetait > y_i \Big | \nitminus + b > M, \ \muitminus \leq x_i, \ \F_{t-1} \right)  \right] \\
     \nonumber \leq & \Pr \left[ \N(\hat{\mu}_{i,t-1}, \frac{c}{n_{i,t-1} + b + 1})  > y_i\bigg | \nitminus + b > M, \ \muitminus \leq x_i, \ \F_{t-1} \right] \\
    \leq & \Pr \left[ \N(x_i, \frac{c}{n_{i,t-1} + b + 1})  > y_i\bigg | \nitminus + b > M, \F_{t-1} \right].
    \label{eq:lem2.16_eq4.0}
\end{align}
The second inequality is from the fact that $x_i \geq \hat{\mu}_{i,t-1}$ and thus $\N(\hat{\mu}_{i,t-1}, \frac{c}{n_{i,t-1} + b + 1})$ is stochastically dominated by $\N(x_i, \frac{c}{n_{i,t-1} + b + 1})$.
To bound the latter expression we will use Mill's inequality (Lemma \ref{lem:Mill's-ineq}) as it allows us to bound the deviations of centered Gaussian random variable $\thetait$ around $0$. 

\begin{lemma}[Mill's inequality] \label{lem:Mill's-ineq}
Let $X\sim N(\mu,\sigma^2)$. Then for any $t>0$,
\begin{equation*}
\Pr[X-\mu > t]\leq \frac{\sigma}{\sqrt{2\pi}}\frac{e^{-\frac{t^2}{2\sigma^2}}}{t} \quad \quad \mbox{and} \quad \quad 
\Pr[X-\mu < -t]\leq \frac{\sigma}{\sqrt{2\pi}}\frac{e^{-\frac{t^2}{2\sigma^2}}}{t}. 
\end{equation*}
\end{lemma}

Using $y_i-x_i=\frac{\Delta_i}{3}$, Mill's inequality (Lemma \ref{lem:Mill's-ineq}), the fact that $\nitminus + b +1> M= \frac{18c \log ((T-bN) \Delta_i ^2)}{\Delta_i^2 }$ and $T\geq bN+\frac{1}{\Delta_i^2}e^{\frac{1}{4\pi}}$  we see that 
\begin{align}
\nonumber \Pr  \left[ \N(x_i, \frac{c}{n_{i,t-1} + b + 1})  > y_i \right] & = \Pr \left[ \N(0, \frac{c}{n_{i,t-1} + b + 1} > \frac13 \Delta_i ) \right] \\
\nonumber 
\nonumber& \leq \sqrt{\frac{c}{2\pi(\nitminus + b + 1)}}\frac{3}{\Delta_i}\exp\left(-\frac{\Delta_i^2}{18}\frac{(\nitminus + b + 1)}{c}\right) \\
\nonumber & \leq  \frac{1}{2\sqrt{\pi\log((T-bN)\Delta_i^2)}}\frac{1}{(T-bN)\Delta_i ^2} \\
\label{eq:lem2.16_eq4}& \leq \frac{1}{(T-bN)\Delta_i ^2}.
\end{align}

We can now use Equations \eqref{eq:lem2.16_eq3}, \eqref{eq:lem2.16_eq4.0}, 
 and \eqref{eq:lem2.16_eq4} and sum all the probabilities over $t = 1, \ldots, T-bN$, yielding  
\begin{equation}
  \label{eq:lem2.16_eq5}
  \sum _{t = 1} ^{T-bN} \Pr \left[a_t = i, \nitminus + b > cM_i(T), E_i ^{\mu} (t), \overline{E_i^{\theta} (t)} \right]\leq \frac{1}{\Delta_i^2}.
\end{equation}
Plugging in the bounds of Equations \eqref{eq:lem2.16_eq2} and \eqref{eq:lem2.16_eq5} into Equation \eqref{eq:lem2.16_eq1} gives the desired bound.
\end{proof}

\thirdtermlem*

\begin{proof}
 Recall that $\pit = \Pr[\theta_{1,t} > y_i | \F_{t-1}]$ and $\thetait \sim \N(\muitminus, \frac{c}{\nitminus + 1})$. Let us introduce some useful notation. Let $\Theta_1,\dots,\Theta_r\overset{iid}{\sim}\N(\hat{\mu}_{1, \tau_k}, \frac{c}{k+b+1})$ be a random sample identically distributed to $\theta_{i,\tau_k}$ given $\F_{\tau_k}$, and let $G_k$ be a geometric random variable that denotes the number of trials until $\Theta_k > y_i$.  
Note that  $p_{i, \tau_k + 1} = \Pr \left[ \Theta_j > y_i | \F_{\tau_k} \right]$ and hence
\begin{align*}
    \E \left[ \frac{1}{p_{i, \tau_k + 1}}\right] -1 = \E \left[ \E \left[ G_k | \F_{\tau_k}\right]\right] = \E\left[ G_k \right]=\sum_{ r=0}^\infty \Pr(G_k\geq r).
\end{align*}

We will therefore establish the desired upper bound by upper bounding the summands of the above expression or equivalently, lower bounding $\Pr(G_k <  j)$. We will first establish a bound that holds for all $k$, and then we will establish a tighter bound that holds when $k$ is sufficiently large.

\paragraph{Upper bound for all $k$:} Define  $\text{MAX}_r : = \max_{1\leq k\leq r}(\Theta_k)$ and $z = \sqrt{ \log r } $. Then, 
\begin{align}
    \Pr [G_k < r] & \geq \Pr [\text{MAX}_r > y_i] \nonumber \\
    & \geq \Pr \left[ \text{MAX}_r > \hat{\mu}_{1, \tau_k} + \sigmac z \geq y_i \right] \nonumber \\
    &=\E \left[ \E \left[ \mathbbm{1}\left(\text{MAX}_r > \hat{\mu}_{1, \tau_k } + \sigmac z \geq y_i \right)\Big| \F_{\tau_k}\right] \right]\nonumber \\
    & = \E \left[ \mathbbm{1} \left(\hat{\mu}_{1, \tau_k } + \sigmac z \geq y_i \right) \Pr \left( \text{MAX}_r > \hat{\mu}_{1, \tau_k} + \sigmac z \Big| \F_{\tau_k} \right)\right]. \label{step.1stterm.const.prod}
\end{align}
The first step holds because $\text{MAX}_r > y_i$ implies $G_k < r$; the second step holds because $\text{MAX}_r > \hat{\mu}_{1, \tau_k} + \sigmac z \geq y_i $ implies $\text{MAX}_r > y_i$; the third step is obtained by taking a double conditional expectation over the history $\F_{\tau_k}$; and the last step separates the two events $\{ \text{MAX}_r > \hat{\mu}_{1, \tau_k} + \sigmac z \}$ and $ \{ \hat{\mu}_{1, \tau_k } + \sigmac z \geq y_i \})$.

To continue bounding this expression, we first lower bound $\Pr ( \text{MAX}_r > \hat{\mu}_{1, \tau_k} +\sigmac z | \F_{\tau_k} )$ using Mill's inequality (Lemma \ref{lem:Mill's-ineq}). 
This lemma gives that for any instantiation $F_{\tau_k}$ of $\F_{\tau_k}$ and $ r>1$,
\begin{align}
\nonumber   \Pr \left[ \text{MAX}_r > \hat{\mu}_{1, \tau_k } + \sigmac z | \F_{\tau_k} = F_{\tau_k} \right] & = 1 - \prod _{k=1}^{r} \Pr \left[\Theta_k \leq \hat{\mu}_{1, \tau_k } + \sigmac z | \F_{\tau_k} = F_{\tau_k} \right]\\ 
\nonumber  & \geq 1 - \left( 1 - \frac{1}{\sqrt{2 \pi }} \frac{e^{-z^2 / 2}}{z} \right)^r \\
\nonumber   & = 1 - \left(  1 - \frac{1}{\sqrt{2\pi r\log r}} \right) ^ r\\ 
    & \geq   1 - e^{- \sqrt{\frac{r}{2 \pi \log r }}}
    \label{step.lem1.1.b0}
\end{align}
The first step applies the definition of $\text{MAX}_r$, and the second step is an application of the Mill's inequality with $t = \sqrt{\frac{c}{k+b+1}}z$ and $\sigma = \sqrt{\frac{c}{k+b+1}}z$. The remaining two steps simplify the expression.

If $r \geq e^{11}$, we have $e ^ {- \sqrt{\frac{r}{2 \pi \log r}}} \leq \frac{1}{r^2}$.
Therefore, we can bound this term separately for $r<e^{11}$ and $r \geq e^{11}$, 
\[
    \Pr \left[ \text{MAX}_r > \hat{\mu}_{1, \tau_k} + \sigmac z | \F_{\tau_k} = F_{\tau_k}\right] \geq \begin{cases} 1 - e^{- \sqrt{\frac{r}{2 \pi \log r }}} & \text{ for }1 < r<e^{11} \\ 1 - \frac{1}{r^2} & \text{ for } r\geq e^{11}
    \end{cases}.
\]

Plugging this into Equation \eqref{step.1stterm.const.prod} gives,
\begin{align}
  \nonumber  \Pr[G_k < r] & \geq \begin{cases}
\left( 1 - e^{- \sqrt{\frac{r}{2 \pi \log r }}} \right) \E \left[ \mathbbm{1} \left(\hat{\mu}_{1, \tau_k } + \sigmac z \geq y_i \right) \right] & \text{ for } 1 < r < e^{11}\\
        \left( 1 - \frac{1}{r^2} \right) \E \left[ \mathbbm{1} \left(\hat{\mu}_{1, \tau_k } + \sigmac z \geq y_i \right) \right] & \text{ for } r \geq e^{11}
    \end{cases} \\
    & = \begin{cases}
\left( 1 - e^{- \sqrt{\frac{r}{2 \pi \log r }}} \right) \Pr \left(\hat{\mu}_{1, \tau_k } + \sigmac z \geq y_i \right) & \text{ for } 1 < r < e^{11}\\
        \left( 1 - \frac{1}{r^2} \right)  \Pr \left(\hat{\mu}_{1, \tau_k } + \sigmac z \geq y_i \right) & \text{ for } r \geq e^{11}
    \end{cases}. \label{step.1stterm.const.prod.final}
\end{align}

Continuing lower bounding this expression, we have:
\begin{align}
\label{eq:Hoeffding3}
\nonumber \Pr \left[ \hat{\mu}_{1, \tau_k }  +\sigmac z  \geq y_i \right] & =\Pr \left[ \hat{\mu}_{1, \tau_k }-\mu_1   \geq -\sqrt{\frac{c\log r}{b+k+1}} -\frac{1}{3}\Delta_i \right] \\
\nonumber &\geq 1 - \exp\left(-2(b+k+1)\left(\frac{1}{3}\Delta_i+\sqrt{\frac{c\log r}{b+k+1}}\right)^2\right)\\
\nonumber &=1 - \frac{1}{r^{2c}}\exp\left(-\frac{2}{9}(b+k+1)\Delta_i^2-\frac{4}{3}\Delta_i\sqrt{(b+k+1)c\log r}\right)\\
\nonumber &\geq 1 - \frac{1}{r^{2c}}\\
&\geq 1 - \frac{1}{r^{2}}
\end{align}
where the first step comes from the definitions of $y_i=\mu_1-\frac{1}{3}\Delta_i$ and $z=\sqrt{\log r}$, the second step comes from an application the Hoeffding's inequality (Lemma~\ref{lem:Hoeffding}) with $a = \frac{1}{3}\Delta_i+\sqrt{\frac{c\log r}{b+k+1}}$ and $n = b+k+1$. The third line expands and combines terms, the fourth line upper bounds the exponential term by 1, and the fifth line follows from $c\geq 1$.

We can return to bounding $\E\left[ G_k \right]=\sum_{r=0}^\infty \Pr(G_k\geq r)$ using Equation \eqref{step.1stterm.const.prod} by plugging in Equations \eqref{step.1stterm.const.prod.final} and \eqref{eq:Hoeffding3}.
\begin{align*}
\E[G_k] & = \sum_{r = 0} ^ {\infty} \Pr [G_k \geq r] \\
& \leq 1+ 1 + \sum_{r = 2} ^ {\infty} (1-\Pr [G_k < r]) \\
&\leq 2+ \sum_{r = 2}^ {\lfloor e^{11}\rfloor} \left( 1- \left(1 - e^{- \sqrt{\frac{r}{2 \pi \log r }}}\right)\left( 1 - \frac{1}{r^2} \right)\right) +\sum_{r = \lceil e^{11}\rceil} ^ {\infty} \left( 1- \left( 1 - \frac{1}{r^2} \right)^2\right) \\
&= 2+ \sum_{r = 2} ^ {\lfloor e^{11}\rfloor} \left(  e^{- \sqrt{\frac{r}{2 \pi \log r }}}\left( 1 - \frac{1}{r^2} \right)+\frac{1}{r^2}\right) +\sum_{r = \lceil e^{11}\rceil} ^ {\infty} \left( \frac{2}{r^2}- \frac{1}{r^4}\right) \\
& \leq 2+2\sum_{r = 2} ^ {\infty} \frac{1}{r^2} +\sum_{r = 2}^{\lfloor e^{11}\rfloor}  e^{- \sqrt{\frac{r}{2 \pi \log r }}}\\
& \leq 2 +\frac{\pi^2}{3}+65.58\\
& \leq 71.
\end{align*}
Thus,
\[\E \left[ \frac{1}{p_{i, \tau_k + 1}} \right] - 1   = \E[G_k]  \leq 71,\]
which is the first upper bound of the lemma.

\paragraph{Tighter upper bound for $k \geq \max\{1,\frac{72c}{\Delta_i^2}\log ((T-bN)\Delta_i^2) - (b+1)\}$:} Note that when $k$ increases, so does the  probability of the event $\theta_1>y_i$ because a larger $k$ makes $\hat\mu_{1,t}$ a more accurate estimate of  $\mu_1$. Using  concentration of $\hat\mu_{1,t}$ and the assumed condition on $k$, we can get much a tighter upper bound on $\mathbb{E}[\frac{1}{p_{i,\tau_k+1}}]$.

We first use the fact that $\Theta_k \sim \N (\hat{\mu}_{1, \tau_k } , \frac{c}{b+k+1})$ to apply Lemma~\ref{lem:as-ineq}, which bounds the tails of $\Theta_k$.

\begin{lemma}\label{lem:as-ineq}
Let $X\sim N(\mu,\sigma^2)$. Then for any $t > 0$, we have 
\begin{equation*}
\Pr[X-\mu > t]\leq e^{-\frac{t^2}{2\sigma^2}} \quad \quad \mbox{and} \quad \quad 
\Pr[X-\mu < -t]\leq e^{-\frac{t^2}{2\sigma^2}} . 
\end{equation*}
\end{lemma}

We instantiate Lemma~\ref{lem:as-ineq} on $\Theta_k$ with $t = \Delta_i / 6$ to get the first inequality below, and use the assumed lower bound on $k$ to get the second inequality below. Thus for any instantiation $F_{\tau_k}$ of $\F_{\tau_k}$,
\begin{align}
\nonumber \Pr\left[\Theta_k > \hat{\mu}_{1, \tau_k} - \frac{\Delta_i}{6} \big| \F_{\tau_k} = F_{\tau_k}\right] 
    &\geq 1 - e^{- \frac{\Delta_i^2 (b+k+1)}{72c}}\\
    &\geq 1-\frac{1}{(T-bN)\Delta_i^2}.
\label{eq:tighter1}
\end{align}

Next define the event $A_{t-1}$ to be the event that $\hat{\mu}_{1, t-1}  - \frac{\Delta_i}{6} \geq y_i$. Note that  $A_{t-1}$ implicitly depends on the history $\F_{t-1}$. Now, consider an instantiation $F_{\tau_k}$ of $\F_{\tau_k}$ such that $A_{\tau_k}$ occurs. For such $F_{\tau_k}$, from Equation \eqref{eq:tighter1} we have that,
\begin{equation}
\label{eq:tighter2}
    \Pr [\Theta_k > y_k | \F_{\tau_k} = F_{\tau_k} ] \geq 1 -  \frac{1}{(T-bN)\Delta_i^2}.
\end{equation}

Let $\F_{t-1} | _{A_{t-1}}$ denote the random variable $\F_{t-1}$ conditioned on the event $A_{t-1}$ occurring. Then
\begin{align}
\nonumber    \E \left[\frac{1}{p_{i, \tau_k + 1}}\right] & = \E \left[\frac{1}{\Pr (\Theta_k > y_i | \F_{\tau_k})}\right] \\
\nonumber    & \leq \E \left[\frac{1}{\Pr (\Theta_k > y_i | \F_{\tau_k} | _{A_{\tau_k}}) \Pr (A_{\tau_k})}\right] \\
 \label{eq:tigher_UB1}   & \leq \E \left[\frac{1}{ (1-\frac{1}{(T-bN)\Delta_i^2}  ) \Pr (A_{\tau_k})} \right]. 
\end{align}
The first step holds because $\Theta_k$ and $\theta_{1, \tau_k}$ have the same distribution. The second step is by law of total probability, and the third step applies the bound in Equation \eqref{eq:tighter2}.

We can continue to bound $\Pr (A_{\tau_k})$ as follows. For any $t \geq \tau_k + 1$ and $j \geq \frac{72c}{\Delta_i^2}\log ((T-bN)\Delta_i^2) - b - 1$,  
\begin{align}
\nonumber\Pr(A_{t-1})&=1-\Pr \left[ \hat{\mu}_{1, t} < \mu_1 - \frac{\Delta_i}{6} \right] \\
\nonumber& \geq 1 - \exp \left( - \frac{ n_{1,t-1} \Delta_i^2}{18}\right) \\
\nonumber & \geq 1 - \exp \left( - 4c \log ((T-bN)\Delta_i^2 ) +\frac{1}{18}\right) \\
&\geq 1-e^{1/18}\frac{1}{(T-bN)^4\Delta_i^8 }
\label{eq:tigher_UB2},
\end{align}
where the first line follows from the definition of $A_{t-1}$, the second line is an application of Hoeffding's inequality (Lemma \ref{lem:Hoeffding}),  the third line uses $n_{1,t-1}\geq b+k\geq \frac{72c}{\Delta_i^2}\log ((T-bN)\Delta_i^2) - 1$ for any $t \geq \tau_k + 1$ and the last inequality used $c\geq 1$.

Hence for  $T\geq bN+\frac{e^{1/54}}{\Delta_i^2}$, from \eqref{eq:tighter2} we  obtain the lower bound
\begin{equation}
    \label{eq:tigher_UB3}
    \Pr(A_{t-1})\geq 1-\frac{1}{(T-bN)\Delta_i^2}.
\end{equation}
Finally, combining \eqref{eq:tigher_UB1}, \eqref{eq:tigher_UB2} and \eqref{eq:tigher_UB3}, using that $(T-bN)\Delta_i^2 \geq 4$ by our assumption, we get that
\begin{equation*}
    \E \left[\frac{1}{p_{i, \tau_k}}\right] -1  \leq \frac{1}{\left(1- \frac{1}{(T-bN)\Delta_i^2}\right)^2} -1 \leq \frac{4}{(T-bN)\Delta_i^2}.
\end{equation*}
\end{proof}

%%%%% %%%%% %%%%% %%%%% %%%%% %%%%%
%%%%%% Appendix C %%%%%
%%%%% %%%%% %%%%% %%%%% %%%%% %%%%%

\section{Alternative Privacy Analyses}\label{app.otherprivacy}

In this appendix, we provide two alternative methods for showing that Thompson Sampling with Gaussian priors is differentially privacy. In the body, we prove this by showing that each timestep $t$ satisfies GDP, then composing under GDP across all timesteps, and finally translating the guarantee back into DP. In Appendix \ref{app.altdp}, we prove this using DP directly, by showing that each timestep satisfies DP, and then applying advanced composition across all timesteps. This results in a more directly interpretable privacy guarantee in terms of the parameters $T$ and $N$, although as we show, it is a substantially weaker guarantee than that achieved using GDP. Along the way, we prove that a more general version of the well-known ReportNoisyMax algorithm \cite{DR14} -- extended to accommodate heterogeneous Gaussian noise, rather than i.i.d.~Laplace noise -- still satisfies DP.

In Appendix \ref{app.altrdp}, we provide another proof, this time using Renyi DP (RDP): we show that each timestep of Thompson sampling satisfies RDP, compose all timesteps under RDP, and then translate back to DP. In Appendix \ref{app.privacycompare} we compare all three of these methods -- GDP, DP, and RDP -- and find that GDP slightly outperforms RDP, and both substantially outperform DP. This is the motivation for focusing our attention on the GDP-based privacy analysis method in the body of the paper.

\subsection{Alternative Method: Standard DP}\label{app.altdp}

In the privacy analysis presented in Sections \ref{s.tsisdp} and \ref{s.improving}, we treat Thompson sampling at every timestep as an instantiation of the Gaussian mechanism and analyze the privacy guarantee assuming that all samples $\{ \theta_{i,t} \}_{i \in [N]}$ are output. However, not all samples need to be published to determine the next arm to pull. Instead, only the index of the sample which with the max value needs to be published:
\begin{align} \label{mech2}
    \arg \max _i \theta_{i,t} \hspace{0.3cm} \text{where} \hspace{0.2cm} \theta_{i,t} \sim \N (\hat{\mu} _{i, t-1} , (\frac{1}{n_{i,t-1} + 1})^2) \hspace{0.2cm} \text{for} \hspace{0.2cm} i = 1, \ldots, N.
 \end{align}
 Intuitively, publishing less information at each timestep should lead to a tighter privacy guarantee for the same amount of noise.

 Fortunately, the ReportNoisyMax algorithm \cite{DR14} (Algorithm \ref{alg.noisymax}) already exists for differentially privately computing an argmax of function values evaluated on a database. This algorithm takes in functions of a fixed sensitivity; it first evaluates all the functions on the database, then adds i.i.d.~Laplace noise to each result, and outputs the argmax of the noisy values. Algorithm~\ref{alg.noisymax} achieves $(\epsilon,0)$-differential privacy by tuning the Laplace noise parameter based on the sensitivity $s$ of all functions and privacy parameter $\epsilon$ \cite{DR14}.

\begin{algorithm}
\begin{algorithmic}[1]
\caption{ReportNoisyMax}\label{alg.noisymax}
\Require Database $R$, real-valued functions $\{f_1, \ldots, f_n\}$ of sensitivity $s$, privacy parameter $\epsilon$
    \State Sample $Z_1, \ldots, Z_n \sim \text{Lap}(s / \epsilon)$ 
    \State Return $\arg \max _{i \in [n]} (f_i(R) + Z_i)$
\end{algorithmic}
\end{algorithm}

To contrast, Equation \eqref{mech2} adds Gaussian noise with different variances to each empirical mean. Thus a single round of Thompson Sampling can be viewed as a variant of the classic ReportNoisyMax algorithm, that adds heterogeneous Gaussian noise, rather than i.i.d.~Laplace noise.
Algorithm~\ref{alg.noisymaxgaus} formally defines the Heterogeneous Gaussian ReportNoisyMax algorithm, which adds Gaussian noise of heterogeneous variances to the results of a set of functions evaluated on the database, where each function is allowed to have different sensitivity. 

\begin{algorithm}
\begin{algorithmic}[1]
\caption{Heterogeneous Gaussian ReportNoisyMax}\label{alg.noisymaxgaus}
\Require Database $R$, real-valued functions $ \{f_1, \ldots, f_n\}$, where $f_i$ has sensitivity $s_i$, noise variances $\{\sigma_1^2,\ldots,\sigma_n^2\}$
    \State Sample $X_i \sim \N(0, \sigma_i^2)$ for $i \in [n]$ 
    \State Return $\arg \max _{i \in [n]} (f_i(R) + X_i)$
\end{algorithmic}
\end{algorithm}

In the context of Thompson sampling, function $f_i(R)$ is the empirical mean $\hat{\mu}_{i,t}$ of arm $i$, given the history of observed rewards $\F_t$. Recall that neighboring histories $\F_t$ and $\F'_t$ contain databases of rewards $R$ and $R'$ that differ only in a single reward observation, so only one arm will have a different empirical mean across these neighbors. Thus we prove differential privacy guarantees for Algorithm~\ref{alg.noisymaxgaus} under the assumption that for any pair of neighboring databases $R, R'$, it holds that $f(R) = (f_1(R), \ldots, f_n(R))$ and $f(R') = (f_1(R'), \ldots, f_n(R'))$ differ at at most one function value.

Theorem \ref{thm-RNM} gives the privacy guarantee of Algorithm~\ref{alg.noisymaxgaus} under this assumption. The proof of this theorem follows closely to the structure of the proof of privacy of ReportNoisyMax in \cite{DR14}, but is modified in key ways based on the change in noise distribution. Additionally, Algorithm~\ref{alg.noisymaxgaus} satisfies $(\epsilon,\delta)$-DP for $\delta>0$ since it uses Gaussian noise, while Algorithm~\ref{alg.noisymax} satisfies $(\epsilon,0)$-DP since it uses Laplace noise.

\begin{theorem} \label{thm-RNM}
Assume that for any pair of neighboring databases $R$ and $R'$, $f(R)$ and $f(R')$ differ at at most one entry.
Then, Algorithm~\ref{alg.noisymaxgaus} is $(\epsilon, \delta)$-differentially private for $\epsilon \geq \frac12 {\sqrt{\log \frac{n-1}{2\delta}}}  \max _{i \in [n]} \left( \frac{s_i}{ \sigma_i} \right)$.
\end{theorem}

\begin{proof}
Fix input functions $f=\{f_1, \ldots, f_n\}$ and noise variances $\{\sigma_1^2,\ldots,\sigma_n^2\}$. Let $R$ and $R'$ be two neighboring databases, let $\M(R)$ denote the output of Algorithm~\ref{alg.noisymaxgaus} on input database $R$, and
let $c = f(R)$ and $c' = f(R')$. Without loss of generality, let $c \geq c'$. This is without loss, since we have assumed that $c$ and $c'$ differ only in a single entry.

Fix any $i \in [n]$. We will bound from above and below the ratio of the probabilities that index $i$ is selected under $R$ versus $R'$. Fix $X_{-i}$ to be a draw from the Gaussian distributions used for all noisy function values except the $i$-th one. 

We first argue that $\Pr [\M(R) = i | X_{-i}] \leq e^{\epsilon} \Pr [\M(R') = i |  X_{-i}] + \frac{1}{n-1} \delta$. Define $X^*$ to be the minimum value of $X_i$ such that $c_i + X_i > c_j' + X_j$ for all $j\neq i$. Then fixing $X_{-i}$, $i$ will the output of Algorithm \ref{alg.noisymaxgaus} on $R$ if and only if $X_i > X^*$. Then for all $j \neq i$, we know that $c_i + X^*  > c_j + X_j$, which also implies: 
\begin{equation*}
     (s_i + c_i' + X^*)  \geq c_i + X^* > c_j + X_j \geq  c_j' + X_j .
\end{equation*}
Thus, if $X_i \geq X^* + s_i$, then the $i$-th noisy function value $f_i(R')$ will be the maximum when the noise vector is $(X_i,X_{-i})$.  

To reason about the probability of the noise term $X_i$ being sufficiently large, we apply a lemma from \cite{NRS07} that bounds the closeness between a standard normal random variable $Z$ and $Z$ plus an additive shift.

\begin{lemma} [\cite{NRS07}] \label{lem-sliding} 
For a subset $\mathcal{S} \in \mathbb{R} ^ d$ and a vector $a \in \mathbb{R} ^ d$, we write $\mathcal{S} + a$ for the set $\{ y + a : y \in \mathcal{S} \}$.
The standard normal distribution $\N(0,1)$ satisfies that for all $\| a \|_1 \leq \frac{2\epsilon}{\sqrt{\log (1/2\delta)}}$ and subsets $\mathcal{S} \in \mathbb{R}^d$, 
$$\Pr _{Z \sim \N(0,1)}[ Z \in \mathcal{S}] \leq e^{\epsilon} \Pr _{Z \sim \N(0,1)}[ Z \in \mathcal{S} + a] + \delta.$$
\end{lemma}

To apply Lemma~\ref{lem-sliding} to our (non-standard) $X_i\sim N(0,\sigma_i^2)$, we use the fact that $X_i = \sigma_i Z$. We instantiate the lemma with $d=1$, $a=\frac{s_i}{\sigma_i}$, and $\mathcal{S}=[X^*,\infty)$. The condition on $|a|$ will be satisfied as long as $\frac{2\epsilon}{\sqrt{\log ((n-1)/2\delta)}} \geq \max _{i \in [n]} \left( \frac{s_i}{ \sigma_i} \right) \geq \frac{s_i}{\sigma_i}$, as required in the theorem statement. Then Lemma~\ref{lem-sliding} gives that:
\[
\Pr [X_i  \geq  X^* | X_{-i} ] \leq e^\epsilon \Pr [X_i \geq X^* + s_i | X_{-i}] + \frac{\delta}{n-1}.
\]

\begin{align*}
 \Pr [\M(R) = i |  X_{-i}] & = \Pr [X_i  \geq  X^* | X_{-i}] \leq e^\epsilon \Pr [X_i \geq X^* + s_i | X_{-i}] + \frac{\delta}{n-1} \leq e^\epsilon \Pr [\M(R') = i| X_{-i}] + \frac{\delta}{n-1}.
\end{align*}

A symmetric argument shows that $\Pr [\M(R')  = i | X_{-i}] \leq e^{\epsilon} \Pr[\M(R) = i | X_{-i}] + \frac{1}{n-1} \delta$. 
Now define $X^*$ to be the minimum value of $X_i$ such that $c_i' + X_i > c_j' + X_j$ for all $j\neq i$. This means that $i$ is the output of $\M$ on $R'$ if and only if $X_i > X^*$. Then for all $j \neq i$, we have 
\begin{align*}
     c_i' + X^* & > c_j' + X_j \\
    \implies \quad s_i + c_i' + X^* & > s_i + c_j' + X_j\\
      \implies \quad c_i + (s_i + X^*)  \geq c_i' + (s_i + X^*)  &> (s_i + c_j') + X_j \geq c_j + X_j.
\end{align*}
Thus, if $X_i \geq X^* + s_i$ then the $i$-th noisy function value will be the maximum under database $R$ when the noise vector is $(X_i,X_{-i})$. 

We again instantiate Lemma~\ref{lem-sliding} in the same way as before to get that,
\[
\Pr [X_i  \geq  X^* | X_{-i}] \leq e^\epsilon \Pr [X_i \geq X^* + s_i | X_{-i}] + \frac{\delta}{n-1}.
\]

Then,
\begin{align*}
 \Pr [\M(R') = i| X_{-i}] = \Pr [X_i  \geq  X^* | X_{-i}] \leq e^\epsilon \Pr [X_i \geq X^* + s_i | X_{-i}] + \frac{\delta}{n-1}  \leq e^\epsilon \Pr [\M(R) = i| X_{-i}] + \frac{\delta}{n-1}.
\end{align*}

Using the law of total probability we can obtain $\Pr[\M(D) = i] \leq e^\epsilon \Pr[\M(D') = i] + \frac{\delta}{n-1}$ for all $i \in [n]$. Then, for any $\mathcal{S} = \{ i_1, i_2, \ldots, i_k\}$, $\Pr[\M(D) \in \mathcal{S} ] \leq e^{\epsilon} \Pr[\M(D') \in \mathcal{S} ] + \frac{k}{n-1} \delta$. The maximum size of $k$ is $n-1$, because when $k = n$, $\mathcal{S} = Range(\M)$, and both probabilities are 1. 
Thus we conclude that Algorithm~\ref{alg.noisymaxgaus} is $(\epsilon, \delta)$-differentially private for $\epsilon \geq \frac12 {\sqrt{\log \frac{n-1}{2\delta}}}  \max _{i \in [n]} \left( \frac{s_i}{ \sigma_i} \right)$.
\end{proof}

Using Theorem \ref{thm-RNM}, we can finally show that Thompson Sampling is differentially private. 

\begin{theorem} \label{thm-RMN-TS} At timestep $t$, the action of Algorithm~\ref{alg.ts.orig} is $(\epsilon, \delta)$-differentially private for $\epsilon = \frac{1}{2\sqrt{2}}  \sqrt{\log \frac{N-1}{2\delta}}$.   
\end{theorem}
\begin{proof}
At timestep $t$ of Algorithm~\ref{alg.ts.orig}, we are instantiating Algorithm~\ref{alg.noisymaxgaus} with sensitivity $s_i = \frac{1}{n_{i,t-1 } + 1}$, and adding Gaussian noise of variance $\sigma^2_{i} = \frac{1}{n_{i,t-1} + 1}$. Plugging these into Theorem~\ref{thm-RNM} gives that Algorithm~\ref{alg.ts.orig} at timestep $t$ is $(\epsilon, \delta)$-differentially private for $\epsilon \geq \frac12 \sqrt{\log \frac{N-1}{2\delta}} \max_i (\frac{1}{n_{i,t-1} + 1} / \sqrt{ \frac{1}{n_{i,t-1} + 1}}) = \frac12 \sqrt{\log \frac{N-1}{2\delta}} \max_{i} \left( {\frac{1}{\sqrt{n_{i,t-1} + 1}}} \right)$. Note that for privacy to be non-trivial, it must be that $n_{i,t-1} \geq 1$, otherwise there would no data to protect. Thus, $\max_{i} \left( {\frac{1}{\sqrt{n_{i,t-1} + 1}}} \right) \leq \frac{1}{\sqrt{2}}$, and the algorithm is $(\epsilon, \delta)$-differentially private for any $\epsilon \geq  \frac{1}{2\sqrt{2}} \sqrt{\log \frac{N-1}{2\delta}} $.
\end{proof}

Corollary~\ref{thm.tsregulardp} gives a bound on the complete Algorithm~\ref{alg.ts.orig} across all timesteps, by utilizing Advanced Composition \cite{Dwork10} to compose the privacy guarantees of Theorem \ref{thm-RMN-TS} across all $T$ timesteps.

\begin{corollary}\label{thm.tsregulardp}
    Given any $\epsilon, \delta$ such that $\epsilon = \frac{1}{2\sqrt{2}} \sqrt{\log \frac{N-1}{2\delta}}$,  Algorithm~\ref{alg.ts.orig} is $(\epsilon_{TS}, \delta_{TS})$-differentially private for $\epsilon_{TS} = \epsilon \sqrt{2T \log (\frac{1}{\delta_{TS} - T\delta} )}  + T\epsilon (e^\epsilon - 1)$ for $\delta_{TS} > T\delta$.
\end{corollary}

\subsection{Alternative Method: RDP}\label{app.altrdp}

Renyi Differential Privacy (RDP) \cite{M17} also generalizes differential privacy, with the guarantees of closeness of outputs across neighboring databases based on \emph{Renyi divergence}.

\begin{definition}
    (Renyi Differential Privacy \cite{M17}). An algorithm $\mathcal{M}$ satisfies $(\alpha, \gamma(\alpha))$-RDP with $\alpha \geq 1$ if for any neighboring datasets $D$ and $D'$:
    \begin{equation*}
        D_\alpha (\mathcal{M}(D)|| \mathcal{M}(D')) = \frac{1}{\alpha -1} \log \mathbbm{E}_{x \sim \mathcal{M}(D)} \left[ \left(\frac{\Pr[\mathcal{M}(D) =x ]}{\Pr[\mathcal{M}(D') =x ]} \right)^{\alpha -1} \right] \leq \gamma(\alpha),
    \end{equation*}
    where the Renyi divergence $D_\alpha$ between two distributions $P$ and $Q$ is 
    \[D_\alpha (P|| Q) = \frac{1}{\alpha -1} \log \mathbbm{E}_{x \sim Q} \left[ (P(x)/ Q(x))^\alpha \right] = \frac{1}{\alpha -1 } \log \mathbbm{E}_{x \sim P} [ (P(x) / Q(x))^{\alpha -1} ].\]
\end{definition}

From \cite{M17}, it is known that the Gaussian Mechanism that adds noise $\N(0, \sigma^2)$ to the value of a function with sensitivity $s$ is $(\alpha, \gamma(\alpha))$-RDP for $\gamma(\alpha) = \frac{ s^2}{2\sigma^2} \alpha$ and for any $\alpha > 1$. In the context of a single-step of the Thompson Sampling algorithm, the sensitivity $s_i$ of function $f_i$ at time $t$ is $\frac{1}{n_{i,t-1} + 1}$, and the variance of the noise added is $\sigma_i^2 = (\frac{1}{n_{i,t-1} + 1})^2$. Let arm $j$ be the arm where the observed rewards differ across two neighboring databases. Then, the RDP guarantee for one pull of this arm is:
\begin{align*}
    \gamma (\alpha) = \frac{s_j^2}{2\sigma_j^2} \alpha =  \frac{1/(n_{j,t-1} + 1 )^2 }{2/(n_{j,t-1} + 1)} \alpha = \frac{1}{2(n_{j,t-1} + 1)}\alpha \leq \frac{1}{4}\alpha.
\end{align*}
The last inequality holds because $n_{j,t-1} \geq 1$, otherwise the dataset would be empty and there would be no data to protect.

The composition guarantees of RDP state that the adaptive composition of $T$ mechanisms that each satisfy $(\alpha, \gamma(\alpha))$-RDP, will together satisfy $(\alpha, T\gamma(\alpha))$-RDP \cite{M17}. Thus $T$ rounds of Thompson Sampling will together satisfy $(\alpha, \frac{1}{4}  \alpha T )$-RDP. 
To convert this RDP guarantee back to DP, we use the fact from \cite{M17} that if an algorithm $\mathcal{M}$ satisfies $(\alpha, \gamma(\alpha))$-RDP, then it also satisfies $(\gamma(\alpha) + \frac{\log (1/ \delta)}{\alpha -1}, \delta)$-DP for any $\delta \in (0,1)$. This gives us the final privacy result of Theorem \ref{thm-rdp}. 

\begin{theorem} \label{thm-rdp}
% Algorithm~\ref{alg.ts.orig} is $(\alpha, \frac{1}{4}  \alpha T ) $-RDP for any $\alpha > 1$. That is, for any $\delta \in (0,1)$ and $\alpha > 1$, it is $(\epsilon, \delta)$-differentially private for $\epsilon = \frac{1}{4}  \alpha T  + \frac{\log (1/\delta)}{\alpha - 1}$.
Algorithm~\ref{alg.ts.orig} is $(\epsilon, \delta)$-differentially private for any $\delta \in (0,1)$, $\alpha > 1$, and $\epsilon = \frac{1}{4}  \alpha T  + \frac{\log (1/\delta)}{\alpha - 1}$.
\end{theorem}

\subsection{Comparisons of the GDP, Standard DP and RDP Results}\label{app.privacycompare}

In this section, we compare the privacy guarantees provided by our three different privacy results: Theorem~\ref{thm.tsdp} (GDP), Corollary~\ref{thm.tsregulardp} (Standard DP), and Theorem \ref{thm-rdp} (RDP). Since the GDP and RDP guarantees are not directly interpretable for comparison, we use empirical evaluations for the comparison.

Figure~\ref{fig:3methods} shows a direct comparison of the numerical privacy guarantees obtained under the three methods: RDP, GDP, and standard DP, for the setting of $T=1,000$ and $N=2,10$ (respectively). We observe that the guarantees obtained by GDP and RDP are significantly tighter than the one obtained by the standard DP method. Unlike the standard DP guarantee, the GDP and RDP guarantees do not depend on $N$. RDP and GDP offer comparable guarantees, with GDP showing a slight advantage in specific regions, particularly when $\delta$ is small, and consistently performing no worse than RDP across all scenarios. This is more clearly observable in Figure~\ref{fig:2methods}, which shows only the GDP and RDP lines on a more zoomed in portion of the y-axis.

 \begin{figure}[ht]
    \centering
\includegraphics[scale=0.5]{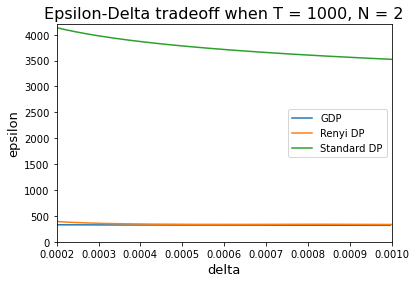}
\includegraphics[scale=0.5]{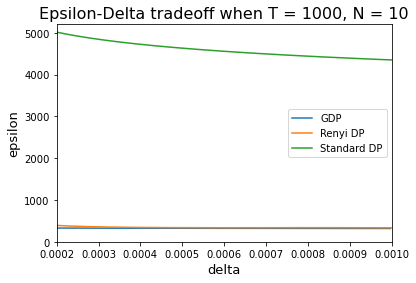}
    \caption{DP parameter $\epsilon$ as a function of $\delta$ when fixing $T = 1000$, obtained by three different analyses (GDP, RDP and Standard DP). The left plot has $N=2$, and the right plot has $N=10$.  }
    \label{fig:3methods}
\end{figure}

 \begin{figure}[ht]
    \centering
\includegraphics[scale=0.5]{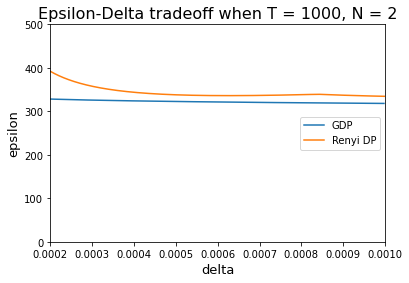}
\includegraphics[scale=0.5]{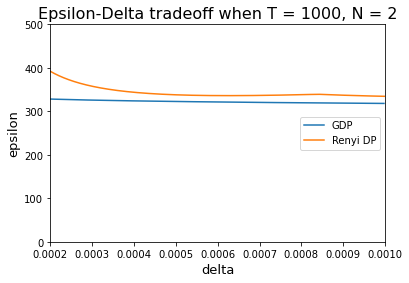}
    \caption{DP parameter $\epsilon$ as a function of $\delta$ when fixing $T = 1000$, obtained by three different analyses (GDP and RDP only). The left plot has $N=2$, and the right plot has $N=10$.  }
    \label{fig:2methods}
\end{figure}

\section{Comparison against Other Private Bandit Algorithms}\label{app.comparison}

We compare the performance of our Modified TS algorithm against two recent non-TS-based DP algorithms for online bandit learning: DP-SE \cite{SS19} and Anytime-Lazy-UCB \cite{HU21}. We use the same arm settings as in Section 5.1, and privacy parameters $\epsilon = 4.88, 35.57, $ and $96.71$ with $\delta=10^{-6}$ for all three algorithms, which corresponds to 1-, 5- and 10-GDP. The $(b,c)$ parameters of Modified Thompson sampling are set via grid search and are omitted on the plots for cleaner presentation.

The empirical results are presented in Figure \ref{fig.compare} below, which shows the regret of all three algorithms over time. We observe that for smaller $\epsilon$ values, the Anytime-Lazy-UCB \cite{HU21} performs the best, and that Anytime-Lazy-UCB always outperforms DP-SE \cite{SS19}. However as $\epsilon$ grows larger, our Modified TS algorithm begins to substantially outperform the other methods.

 \begin{figure}[h]
 \centering
\includegraphics[scale=0.5]{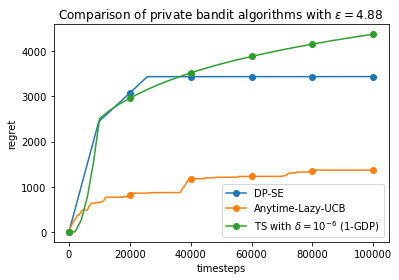}
\includegraphics[scale=0.5]{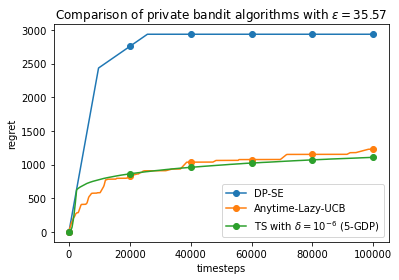}
\includegraphics[scale=0.5]{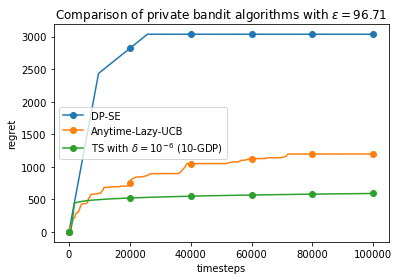}
\caption{Comparison of our Modified TS algorithm with DP-SE \cite{SS19} and Anytime-Lazy-UCB \cite{HU21}, under fixed privacy levels $\epsilon = 4.88, 35.57, $ and $96.71$. }\label{fig.compare}
 \end{figure}

\end{document}